\documentclass{article}
  
\usepackage[letterpaper, margin=1.1in]{geometry}

\usepackage{graphicx,xspace,xcolor} 

\usepackage{times}
\usepackage{soul}
\usepackage{url}
\usepackage[hidelinks]{hyperref}
\usepackage[utf8]{inputenc}
\usepackage[small]{caption}
\usepackage{graphicx}
\usepackage{amsmath}
\usepackage{amsthm}
\usepackage{booktabs}
\usepackage{algorithm}
\usepackage[sort&compress, numbers]{natbib}

\usepackage[noend]{algpseudocode}
\usepackage{algorithmicx}

\algnewcommand\algorithmicinput{\textbf{Input:}}
\algnewcommand\INPUT{\item[\algorithmicinput]}

\algnewcommand\algorithmicoutput{\textbf{Output:}}
\algnewcommand\OUTPUT{\item[\algorithmicoutput]}

\newcommand{\NULL}{\textnormal{\texttt{nil}}}

\newif\iflong
\newif\ifshort

\longtrue

\iflong
\else
\shorttrue
\fi

\usepackage{cleveref}

\newtheorem{theorem}{Theorem}
\newtheorem{lemma}[theorem]{Lemma}

\newtheorem{proposition}[theorem]{Proposition}
\newtheorem{corollary}[theorem]{Corollary}

\usepackage{amssymb}
\usepackage{enumerate,verbatim}
\usepackage{xspace}
\usepackage{tikz,tikz-cd}
\usetikzlibrary{arrows,cd,positioning,shapes,patterns}
\usetikzlibrary {decorations.pathmorphing, decorations.pathreplacing, decorations.shapes}
\usepackage{subcaption}
\usetikzlibrary{calc}
\tikzset{cross/.style={cross out, draw=black, minimum size=2*(#1-\pgflinewidth), inner sep=0pt, outer sep=0pt},
cross/.default={1pt}}

\usepackage{todonotes}
\presetkeys%
    {todonotes}%
    {backgroundcolor=yellow}{}

\newcommand{\bigoh}{\mathcal{O}}

\title{Explaining Decisions in ML Models:\\ a Parameterized Complexity
  Analysis (Part I)\thanks{This is Part I of
  a greatly extended version of~\cite{DBLP:conf/kr/OrdyniakPRS24} (see
  also the full version of~\cite{DBLP:conf/kr/OrdyniakPRS24} on arXiv~\cite{arxiv-version}), which covers
  all results concerning decision trees, decision sets, and decision
  lists. Due to the length and the large number of
  results in the original conference paper, which grew even larger when preparing the journal version,
  we decided to split the paper into two parts for the convenience of
  the reviewers. Compared to~\cite{DBLP:conf/kr/OrdyniakPRS24} the
  current paper also entails several
  major improvements such as: (1) A greatly improved presentation of our algorithmic
  meta-theorem for Boolean Circuits (Subsection 5.1), (2) a much improved presentation of the hardness results with the help
  of the introduction of a novel meta-framework (Subsection 6.1) for
  showing hardness, (3) a strengthening of the W$[1]$-hardness result for local contrastive
  explanations for DSs (and DLs) parameterized by explanation size to
  W$[2]$-hardness, (4) a completed landscape for ordered DTs showing that they
  behave in the same way as normal DTs with respect to all explanation
  problems and parameters, which is an important prerequisite for the
  analysis of BDDs in Part II. Throughout this paper, we will refer to the yet unpublished second
  part of this paper, which covers results for various variants of binary decision
  diagrams, as Part II.}
}
\date{}
\author{Sebastian Ordyniak\thanks{School of Computing, University of Leeds, UK,
    \texttt{sordyniak@gmail.com}} \and
  Giacomo Paesani\thanks{Dipartimento di Informatica, Sapienza University of Rome, Italy, \texttt{giacomopaesani@gmail.com}} \and  
  Mateusz Rychlicki\thanks{School of Computing, University of Leeds, UK,
    \texttt{mkrychlicki@gmail.com}} \and
Stefan Szeider\thanks{Algorithms and Complexity Group, TU Wien,
  Vienna, Austria, \texttt{sz@ac.tuwien.ac.at}}}

\usepackage{boxedminipage}
\newcommand{\pbDef}[3]{%
  \noindent
  \begin{center}
  \begin{boxedminipage}{0.98 \columnwidth}
  {\sc #1}\\[2pt]
  \begin{tabular}{@{}l@{~~} p{0.76 \columnwidth}@{}}
  {\sc Instance}: & #2\\
  {\sc Question}: & #3
  \end{tabular}
  \end{boxedminipage}
  \end{center}
}

\newcommand{\mypara}[1]{\smallskip\noindent\textbf{#1}}

\newcommand{\cc}[1]{{\mbox{\textnormal{\textsf{#1}}}}\xspace}  

\newcommand{\hy}{\hbox{-}\nobreak\hskip0pt}

\renewcommand{\P}{\cc{P}}
\newcommand{\NP}{\cc{NP}}

\newcommand{\FPT}{\cc{FPT}}
\newcommand{\XP}{\cc{XP}}
\newcommand{\Weft}{{\cc{W}}}
\newcommand{\W}[1]{{\Weft}{{\textnormal[\ensuremath{#1}\textnormal]}}}
\newcommand{\coW}[1]{{\mbox{\textnormal{\textsf{co-}}}\Weft}{{\textnormal[\ensuremath{#1}\textnormal]}}}
\newcommand{\paraNP}{\cc{paraNP}}

\newcommand{\poly}{\textit{poly}}
\newcommand{\NPh}{{\mbox{\textnormal{\textsf{NP}}}}\hy{}\text{hard}}
\newcommand{\coNPh}{{\mbox{\textnormal{\textsf{co-NP}}}}\hy{}\text{hard}}
\newcommand{\Wh}[1]{\W{#1}\hy{}\text{hard}}
\newcommand{\coWh}[1]{\coW{#1}\hy{}\text{hard}}
\newcommand{\pNPh}{{\mbox{\textnormal{\textsf{pNP}}}}\hy{}\text{hard}}
\newcommand{\pNPtable}{{\mbox{\textnormal{\textsf{pNP}}}}\hy{}\text{h}}
\newcommand{\copNPtable}{{\mbox{\textnormal{\textsf{co-pNP}}}}\hy{}\text{h}}

\newcommand{\NPtable}{{\mbox{\textnormal{\textsf{NP}}}}\hy{}\text{h}}
\newcommand{\coNPtable}{{\mbox{\textnormal{\textsf{co-NP}}}}\hy{}\text{h}}
\newcommand{\Wtable}[1]{\W{#1}\hy{}\text{h}}
\newcommand{\coWtable}[1]{\coW{#1}\hy{}\text{h}}

\newcommand{\SM}{{\;{|}\;}}
\newcommand{\SE}{\,\}}
\newcommand{\SB}{\{\,}

\newcommand{\MM}{\ensuremath{\mathcal{M}}}
\newcommand{\PP}{\ensuremath{\mathcal{P}}}
\newcommand{\FFF}{\ensuremath{\mathcal{F}}}
\newcommand{\BBB}{\ensuremath{\mathcal{B}}}
\newcommand{\TTT}{\ensuremath{\mathcal{T}}}
\newcommand{\MMM}{\ensuremath{\mathcal{E}}}
\newcommand{\PPP}{\mathcal{P}}

\newcommand{\modelname}[1]{\textnormal{#1}}
\newcommand{\MAJ}{\textnormal{MAJ}}

\newcommand{\DT}{\modelname{DT}\xspace}

\newcommand{\DTO}{\modelname{ODT}\xspace}
\newcommand{\DTOEO}{\modelname{ODT}\ensuremath{_\MAJ^<}\xspace}
\newcommand{\DS}{\modelname{DS}\xspace}
\newcommand{\DL}{\modelname{DL}\xspace}
\newcommand{\DSE}{\modelname{DS}\ensuremath{_\MAJ}\xspace}
\newcommand{\DLE}{\modelname{DL}\ensuremath{_\MAJ}\xspace}

\newcommand{\NN}{\modelname{NN}\xspace}

\newcommand{\RF}{\modelname{DT\ensuremath{_\MAJ}}\xspace}
\newcommand{\RFO}{\DTOEO}
\newcommand{\BC}{\modelname{BC}\xspace}

\newcommand{\ds}{S}

\newcommand{\dl}{L}
\newcommand{\dle}{\mathcal{L}}

\newcommand{\HOM}{\textsc{HOM}}
\newcommand{\kHOM}{\textsc{p-HOM}}

\renewcommand{\SS}{\subseteq}
\newcommand{\CD}{{\rule{0.5pt}{1.2ex}~\rule{0.5pt}{1.2ex}}}

\newcommand{\N}{$-$\xspace}

\newcommand{\LAEX}{\textsc{lAXp}}
\newcommand{\GAEX}{\textsc{gAXp}}
\newcommand{\LCEX}{\textsc{lCXp}}
\newcommand{\GCEX}{\textsc{gCXp}}

\newcommand{\SMLAEX}{\textsc{lAXp${}_\SS$}}
\newcommand{\SMGAEX}{\textsc{gAXp${}_\SS$}}
\newcommand{\SMLCEX}{\textsc{lCXp${}_\SS$}}
\newcommand{\SMGCEX}{\textsc{gCXp${}_\SS$}}

\newcommand{\MLAEX}{\textsc{lAXp${}_\CD$}}
\newcommand{\MGAEX}{\textsc{gAXp${}_\CD$}}
\newcommand{\MLCEX}{\textsc{lCXp${}_\CD$}}
\newcommand{\MGCEX}{\textsc{gCXp${}_\CD$}}

\newcommand{\SPP}{\ensuremath{\PP_\SS}}
\newcommand{\MPP}{\ensuremath{\PP_\CD}}

\newcommand{\MCC}{\textsc{MCC}}

\newcommand{\sizeelem}{\textsl{size\_elem}\xspace}
\newcommand{\termselem}{\textsl{terms\_elem}\xspace}
\newcommand{\enssize}{\textsl{ens\_size}\xspace}
\newcommand{\termsize}{\textsl{term\_size}\xspace}

\newcommand{\xpsize}{\textsl{xp\_size}\xspace}
\newcommand{\mnlsize}{\textsl{mnl\_size}\xspace}

\newcommand{\MNL}{\textsf{MNL}}
\newcommand{\CIRC}{\mathcal{C}}
\newcommand{\concat}{\circ}
\newcommand{\som}[1]{\|#1\|}

\newcommand{\feat}{F}

\newcommand{\EEE}{\mathcal{E}}

\newcommand{\rw}{\textsf{rw}}

\usepackage{adjustbox}
\usepackage{array}

\newcolumntype{R}[2]{%
    >{\adjustbox{angle=#1,lap=\width-(#2)}\bgroup}%
    l%
    <{\egroup}%
}
\newcommand*\rot{\multicolumn{1}{R{20}{1.8em}}}

\begin{document}
\maketitle

\begin{abstract}
   This paper presents a comprehensive theoretical investigation into
  the parameterized complexity of explanation problems in various
  machine learning (ML) models. Contrary to the prevalent black-box
  perception, our study focuses on models with transparent internal
  mechanisms. We address two principal types of explanation problems:
  abductive and contrastive, both in their local and global
  variants. Our analysis encompasses diverse ML models, including
  Decision Trees, Decision Sets, Decision Lists, Boolean Circuits, and ensembles
  thereof, each offering unique explanatory challenges. This research fills a
  significant gap in explainable AI (XAI) by providing a foundational
  understanding of the complexities of generating explanations for
  these models.  This work provides insights vital for further
  research in the domain of XAI, contributing to the broader discourse
  on the necessity of transparency and accountability in AI systems.
\end{abstract}

\section{Introduction}

When a machine learning model denies someone a loan, diagnoses a medical condition, or filters job applications, we naturally want to know why. This fundamental need for transparency has become even more pressing as AI systems move into high-stakes domains like healthcare, criminal justice, and financial services. Recent regulatory frameworks, including the EU's AI Act and various national guidelines, now mandate explainability for automated decision-making systems \cite{EU20,OECD23}. Yet despite this urgent practical need, we still lack a rigorous theoretical foundation for understanding which models can be efficiently explained and under what conditions.

This paper tackles a deceptively simple question: how hard is it, computationally, to generate explanations for different types of machine learning models? While practitioners have developed numerous tools and techniques for explaining model decisions \cite{Guidotti-etal-2018,Carvalho-etal-19}, the underlying computational complexity of these explanation tasks remains poorly understood. This gap matters because it determines which explanation approaches can scale to real-world systems and which are doomed to computational intractability.

We focus on models whose internal workings are fully
accessible---decision trees, decision sets, decision lists, and their
ensembles (like random forests). Unlike neural networks, these models
do not hide their decision logic behind layers of opaque computations \cite{Ribeiro0G16,Lipton18}. One might assume this transparency makes them easy to explain. Our results reveal a more nuanced picture: even for these ``interpretable'' models, generating certain types of explanations can be computationally prohibitive, while others admit efficient algorithms under the right conditions.

\subsection*{Our Approach}

We analyse four fundamental types of explanations that capture
different aspects of model behaviour \cite{Silva22}. Local explanations
focus on individual predictions: an \emph{abductive} explanation
identifies which features were sufficient for a particular decision
(answering: what features made the model classify this example as
positive?) \cite{IgnatievNM19,ShihCD18,DarwicheJi22,DBLP:conf/aaai/HuangIICA022}, while a
\emph{contrastive} explanation reveals what would need to change for a
different outcome (answering: what minimal changes would flip this prediction?) \cite{Miller19,IgnatievNA20,DarwicheJi22}. Global explanations characterise model behaviour across all possible inputs by identifying feature patterns that guarantee a certain classification (abductive) or prevent it (contrastive) \cite{Ribeiro0G16,IgnatievNM19}.

Since most of these explanation problems turn out to be NP-hard in
general, we employ the framework of parameterized complexity
\cite{DowneyFellows13}. This approach lets us identify specific
aspects of the problem (like the size of the explanation, the number
of rules in a model, or structural properties of decision trees) that can be exploited to obtain efficient algorithms for restricted but practically relevant cases. A problem that seems hopelessly intractable in general might become efficiently solvable when we know the explanations will be small or the models have bounded complexity along certain dimensions.

\subsection*{Technical Contributions}

Our analysis yields positive and negative results, painting a detailed
complexity landscape across different model types and explanation
problems. On the algorithmic side, we develop a powerful meta-theorem
based on Boolean circuits and monadic second-order logic that yields
fixed-parameter tractable algorithms for a wide range of explanation
problems. This framework is particularly notable because it extends
naturally to handle ensemble models by incorporating majority
gates---a technical challenge that required extending standard MSO
logic with counting capabilities \cite{BergougnouxDJ23}.
In Part II of this paper, this same framework also proves instrumental for analysing binary decision
diagrams and other circuit-based models, demonstrating its broad
applicability.

We complement these algorithms with an equally general framework for proving hardness results. A key insight underlying many of our hardness proofs is the close connection between explanation problems and the homogeneous model problem---deciding whether a model classifies all examples similarly. This connection allows us to transfer hardness results from this fundamental problem to various explanation scenarios.

More significantly, we introduce a novel meta-framework in Section 6.1
that unifies the majority of our hardness proofs. This framework
identifies two crucial properties that make model classes amenable to
hardness reductions:
\begin{itemize}
\item \emph{Set-modelling}: A model class can efficiently construct models that classify an example positively if and only if its features correspond exactly to some set in a given family of sets.
\item \emph{Subset-modelling}: A model class can efficiently construct models that classify an example positively if its features contain (as a subset) some set in a given family.

\end{itemize}

These properties might seem technical, but they capture something
fundamental about a model's expressive power. When a model class
satisfies either property, we can encode hard combinatorial problems
(like finding a maximum clique or a minimum hitting set) into the task of explaining the
model's behaviour. This meta-theorem covers the vast majority of
hardness results in this paper, applying to single models and ensembles alike. The broad applicability of this framework suggests it will prove valuable for analysing the complexity of explanation problems in further model classes as well.


\subsection*{Broader Impact}

Beyond the specific results, which we provide in detail in Section~\ref{sec:overview}, this work makes several contributions to
the foundations of explainable AI. First, it provides a rigorous
framework for comparing the ``explainability'' of different model
classes based on the computational complexity of generating
explanations. A model is not just interpretable or not---different
models support different types of explanations with varying
computational costs. Second, our meta-theorems offer general tools
that can be applied to analyse new model classes as they emerge.
The frameworks we develop here have already proven useful for
analysing binary decision diagrams in Part II of this paper, and we expect them to find further applications.

Our results also have practical implications for the design of explainable AI systems. When choosing between different interpretable models, practitioners can now make informed decisions about the computational costs of explanation generation. For instance, while decision sets might seem appealing for their simple if-then rules, our results show that generating certain explanations for them is computationally harder than for decision trees. Similarly, our parameterized algorithms identify exactly when explanation problems become tractable, guiding the design of efficient explanation systems for real-world applications.

This work builds on the pioneering perspective introduced by \citet{BarceloM0S20}, who first proposed quantifying model explainability through computational complexity. While \citet{OrdyniakPS23} recently examined the parameterized complexity of finding explanations from black-box samples, we take a complementary approach by assuming full access to the model's internal structure. This allows us to provide a more complete characterisation of when and why explanation generation becomes computationally difficult.

The increasing emphasis on interpretable ML \cite{LisboaSVFV23} and the call for theoretical rigour in AI \cite{EU19} make this work particularly timely. Prior research has largely focused on practical explainability approaches, but a comprehensive theoretical understanding has been lacking~\cite{HolzingerSMBS20,Molnar23}. By dissecting the parameterized complexity of explanation problems, we lay the groundwork for future research and algorithm development, ultimately contributing to more efficient explanation methods in AI.

This paper is organised as follows.
After establishing preliminaries in Section~\ref{sec:preliminaries}, we formally define the explanation problems and parameters in Section~\ref{sec:problems}. Section~\ref{sec:overview} provides an overview of our results through comprehensive complexity tables. Sections~\ref{sec:algorithmic} and~\ref{sec:hardness} present our algorithmic and hardness results respectively, with Section~\ref{ssec:bc} introducing the Boolean circuit meta-theorem and Section~\ref{ss:meta_hardness} presenting the hardness framework based on set-modelling and subset-modelling properties. We conclude with a discussion of implications and future directions in Section~\ref{sec:conclusion}.
\section{Preliminaries}\label{sec:preliminaries}

For a positive integer $i$, we denote by $[i]$ the set of integers
$\{1,\dotsc,i\}$.

\subsection{Parameterized Complexity (PC).}  We outline some basic concepts
refer to the textbook by \citet{DowneyFellows13} for an in-depth
treatment. An instance of a parameterized problem $Q$ is a pair
$(x,k)$ where $x$ is the main part and $k$ (usually an non-negative
integer) is the parameter. $Q$ is \emph{fixed-parameter tractable
  (fpt)} if it can be solved in time $f(k)n^c$, where $n$ is the input
size of $x$, $c$ is a constant, and $f$ is a
computable function. If a problem has more then one parameters, then
the parameters can be combined to a single one by addition.  $\FPT$
denotes the class of all fixed-parameter tractable decision
problems. $\XP$ denotes the class of all parameterized decision
problems solvable in time $n^{f(k)}$ where $f$ is again a computable
function. An \emph{fpt-reduction} from one parameterized decision
problem $Q$ to another $Q'$ is an fpt-computable reduction that
reduces of $Q$ to instances of $Q'$ such that yes-instances are mapped
to yes-instances and no-instances are mapped to no-instances. The
parameterized complexity classes $\W{i}$ are defined as the closure of
certain weighted circuit satisfaction problems under
fpt-reductions. Denoting by $\P$ the class of all parameterized
decision problems solvable in polynomial time, and by $\paraNP$ 
the
class of parameterized decision problems that are in $\NP$ 
for every instantiation of the parameter with a constant, we
have
$ \P \subseteq \FPT \subseteq \W{1} \subseteq \W{2} \subseteq \cdots
\subseteq \XP \cap \paraNP \subseteq \paraNP$, where all inclusions
are believed to be strict. If a parameterized problem is $\Wh{i}$ under fpt-reductions ($\Wtable{i}$, for short) then it is unlikely
to be fpt. $\textsf{co-C}$ denotes the complexity class containing all
problems from $\textsf{C}$ with yes-instances replaced by no-instances and
no-instances replaced by yes-instances.

\subsection{Graphs, Rank-Width, Treewidth and Pathwidth.}
We mostly use standard notation for graphs as can be found, e.g., in~\cite{Diestel00}.
Let $G=(V,E)$ be a directed or undirected graph. For a vertex subset
$V' \subseteq V$, we denote by $G[V']$ the graph induced by the
vertices in $V'$ and by $G \setminus V'$ the graph $G[V \setminus
V']$. 
For a graph $G$, we denote we denote by $N_G(v)$ the set of all neighbours
of the vertex $v \in V$.
If $G$ is directed, we denote by $N_G^-(v)$ ($N_G^+(v)$) the set
of all incoming (outgoing) neighbours of $v$.

Let $G=(V,E)$ be a directed graph. A
\emph{tree decomposition} of $G$ is a pair $\TTT=(T,\lambda)$ with
$T$ being a tree and $\lambda : V(T) \rightarrow 2^{V(G)}$
such that: (1)
for every vertex $v \in V$ the set $\SB t \in V(T) \SM
v \in \lambda(t)\SE$ is forms a non-empty  subtree
of $T$ and (2)
for every arc $e=(u,v) \in E$, there is a node $t \in V(T)$ with $u,v
\in \lambda(t)$. The \emph{width} of $\TTT$ is equal to $\max_{t \in
  V(T)}|\lambda(t)|-1$ and the \emph{treewidth} of $G$ is the minimum
width over all tree decompositions of $G$. $\TTT$ is called a
\emph{path decomposition} if $T$ is a path and the \emph{pathwidth} of
$G$ is the minimum width over all path decompositions of $G$. We will
need the following well-known properties of pathwidth, treewidth, and
rank-width; since it is not needed for
our purposes, we will not define rank-width but only state its
relationship to other parameters (see, e.g.~\cite{DBLP:journals/jct/OumS06}, for a definition of rank-width).
\begin{lemma}[{\cite{DBLP:conf/wg/CorneilR01,DBLP:journals/jct/OumS06}}]\label{lem:ranktree}
  Let $G=(V,E)$ be a directed graph and $X \subseteq V$. The
  treewidth of $G$ is at most $|X|$ plus the treewidth of
  $G-X$. Furthermore, if $G$ has rank-width $r$, pathwidth $p$ and
  treewidth $t$, then
  $r \leq 3\cdot 2^{t-1}\leq 3\cdot 2^{p-1}$. 
\end{lemma}

\subsection{Examples and Models}
Let $F$ be a set of binary features. An \emph{example} $e : F
\rightarrow \{0,1\}$ over $F$ is a $\{0,1\}$-assignment of the features
in $F$. An example is a \emph{partial example (assignment)} over $F$
if it is an example over some subset $F'$ of $F$. We denote by $E(F)$
the set of all possible examples over $F$. A \emph{(binary
  classification) model} $M : E(F) \rightarrow \{0,1\}$ is a specific
representation of a Boolean function over $E(F)$. We denote by
$\feat(M)$ the set of features considered by $M$, i.e., $\feat(M)=F$.
We say that an example $e$ is a $0$-example or negative example ($1$-example
or positive example) w.r.t. the model $M$ if $M(e)=0$ ($M(e)=1$).
For convenience, we restrict our setting to the classification into two
classes. We note however that all our hardness results easily carry over to the
classification into any (in)finite set of classes. The same applies to
our algorithmic results for non-ensemble models since one can easily
reduce to the case with two classes by renaming the class of interest
for the particular explanation problem to $1$ and all other classes to
$0$. We leave it open whether the same holds for our algorithmic
results for ensemble models.

\subsection{Decision Trees.}
A \emph{decision tree} ($\DT$) $\TTT$ is a pair $(T,\lambda)$
such that $T$ is a
rooted binary tree and $\lambda : V(T) \rightarrow F \cup \{0,1\}$ is
a function that assigns a feature in $F$ to every inner node of $T$
and either $0$ or $1$ to every leaf node of $T$. Every inner node of
$T$ has exactly $2$ children, one left child (or $0$-child) and one
right-child (or $1$-child).
The classification function $\TTT : E(F) \rightarrow \{0,1\}$ of a \DT{}
is defined as follows for an example $e \in E(F)$. 
Starting at the root of $T$ one does the following at
every inner node $t$ of $T$. If $e(\lambda(t))=0$ one continues with
the $0$-child of $t$ and if $e(\lambda(t))=1$ one continues with the
$1$-child of $t$ until one eventually ends up at a leaf node $l$ at which
$e$ is classified as $\lambda(l)$.
For every node $t$ of $T$, we denote by $\alpha_{\TTT}^t$ the
partial assignment of $F$ defined by the path from the root of $T$ to
$t$ in $T$, i.e., for a feature $f$, we set $\alpha_{\TTT}^t(f)$ to $0$ ($1$) if
and only if the path from the root of $T$ to $t$ contains an inner
node $t'$ with $\lambda(t')=f$ together with its $0$-child
($1$-child).
We denote by $L(\TTT)$ the set of leaves of $\TTT$,  
and for every $b \in \{0,1\}$, we let  
$L_b(\TTT) = \SB l \in L(\TTT) \SM \lambda(l) = b \SE$.  
We denote by $\som{\TTT}$ the size of $\TTT$, that is, the number of leaves of $\TTT$.  
Finally, we denote by $\MNL(\TTT)$ the minimal number of leaves of the
same kind, that is, $\MNL(\TTT) = \min\{ |L_0(\TTT)|, |L_1(\TTT)| \}$.
Let $<$ be a total ordering of features and let $\mathcal{T}$ be a
\DT. We say that a $\mathcal{T}$ \emph{respects} $<$ if so does every
root-to-leaf path of $\mathcal{T}$.
We additionally consider the class of all ordered decision trees,  
denoted by $\DTO$, i.e., the class of all \DT{}s $\mathcal{T}$ for which
there is a total ordering $<$ of all features of $\mathcal{T}$ such
that $\mathcal{T}$ respects $<$. 

\subsection{Decision Sets.}
A \emph{term} $t$ over $F$ is a set of
\emph{literals} with each literal being
of the form $(f=z)$ where $f\in F$ and $z\in \{0,1\}$. A \emph{rule}
$r$ is a pair $(t,c)$ where $t$ is a term and $c\in \{0,1\}$. We say
that a rule $(t,c)$ is a \emph{$c$-rule}.
We say that a term $t$ (or rule $(t,c)$) \emph{applies to (or agrees
  with)} an example $e$ if $e(f)=z$ for every element $(f=z)$ of $t$.
Note that the empty rule applies to any example.

A {\it decision set} ($\DS$) $\ds$ is a pair $(T,b)$, where $T$ is a set of terms
and $b \in \{0,1\}$ is the classification of the default rule (or the
default classification).
We denote by $\som{\ds}$ the size of $\ds$,  
defined as $(\sum_{t \in T} |t|) + 1$, where the $+1$ accounts for the default rule.  
We denote by $|\ds|$ (also referred to as $\termselem{}$)  
the number of terms in $\ds$, that is, $|T|$.
The classification function $\ds : E(F) \rightarrow \{0,1\}$  
of a disjunctive specification $\ds = (T, b)$  
is defined as follows: for every example $e \in E(F)$,  
we set $\ds(e) = b$ if no term in $T$ applies to $e$,  
and $\ds(e) = 1 - b$ otherwise.

\subsection{Decision Lists.}
A {\it decision list} ($\DL$) $L$ is a non-empty sequence of rules
$(r_1=(t_1,c_1),\dotsc,r_\ell=(t_\ell,c_\ell))$, for some $\ell \geq 0$.
The size of a \DL{} $L$, denoted by $\som{L}$, is defined as  
$\sum_{i=1}^\ell (|t_i| + 1)$.  
We denote by $|L|$ the number of terms in $L$ (i.e., the \termselem{}), which equals $\ell$.
The classification function $L : E(F) \rightarrow \{0,1\}$ of a \DL{} $L$  
is defined such that $L(e) = b$ if the first rule in $L$ that applies to $e$ is a $b$-rule.  
To guarantee that every example receives a classification,  
we assume that the term of the last rule is empty,  
and thus applies to all examples.

\subsection{Ensembles}  
An \emph{$\MM$-ensemble}, also denoted by $\MM_\MAJ$,  
is a set $\EEE$ of models of type $\MM$,  
where $\MM \in \{\DT, \DTO{}, \DS, \DL\}$.  
We say that $\EEE$ classifies an example $e \in E(F)$ as $b$  
if the majority of models in $\EEE$ classify $e$ as $b$,  
that is, if at least $\lfloor |\EEE|/2 \rfloor + 1$ models  
in $\EEE$ assign $b$ to $e$.  
We denote by $\som{\EEE}$ the size of $\EEE$,  
defined as $\sum_{M \in \EEE} \som{M}$.
We additionally
consider an \emph{ordered ensemble of $\DTO$s}, denote by \DTOEO{},
where all \DTO{}s in the ensemble respect the same ordering of the features.


\section{Considered Problems and Parameters}\label{sec:problems}


\begin{figure}
\centering
\begin{tikzpicture}[scale=1]
\tikzstyle{q} = [draw, rectangle,  fill= gray!10, inner sep=2pt]
\tikzstyle{t} = [draw=black,rectangle, fill= darkgray, text=white, inner sep=2pt]
\draw (13.0,-2.5) node  (dl) {\fbox{$\arraycolsep=1.4pt
\begin{array}{llll}
r_1:&\text{IF}   & (x=1 \wedge y=1)&\text{THEN } 0\\
r_2:&\text{ELSE IF}   &(x=0 \wedge z=0) &\text{THEN } 1\\
r_3:&\text{ELSE IF}   &(y=0 \wedge z=1) &\text{THEN } 0\\
r_4:&\text{ELSE}   &  &\text{THEN } 1\\
\end{array}
$}};
\end{tikzpicture}
\caption{
Let $L$ be the \DL given in the figure and let $e$ be the example given by $e(x)=0$, $e(y)=0$ and $e(z)=1$. Note that $L(e)=0$. It is easy to verify that $\{y,z\}$ is the only local abductive explanation for $e$ in $L$ of size at most 2. Moreover, both $\{y\}$ and $\{z\}$ are minimal
local contrastive explanations for $e$ in $L$.
Let $\tau_1=\{x\mapsto 1,y \mapsto 1\}$ and $\tau_2=\{ x \mapsto 0,z \mapsto 0\}$ be a partial assignments.
Note that $\tau_1$ and $\tau_2$ are minimal global abductive and global contrastive
explanations for class $0$ w.r.t. $L$, respectively.
}
\label{fig:expl}
\end{figure}

We consider the following types of explanations
(see~\citeauthor{Silva22}'s  survey \cite{Silva22}).
Let $M$ be a model, $e$ an example over $\feat(M)$, and let $c \in
\{0,1\}$ be a classification (class). We consider the following types
of explanations for which an example is illustrated in \Cref{fig:expl}.
\begin{itemize}
\item A \emph{(local) abductive explanation (\LAEX{}) for $e$
  w.r.t. $M$} is a subset $A \subseteq \feat(M)$ of features
  such that $M(e)=M(e')$ for every example $e'$ that agrees with $e$ on $A$.
\item A \emph{(local) contrastive explanation (\LCEX{}) for $e$
    w.r.t. $M$} is a set $A$ of
  features such that there is an example $e'$ such that $M(e')\neq
  M(e)$ and $e'$ differ from $e$ only on the features in $A$.
\item A \emph{global abductive explanation (\GAEX{}) for $c$
  w.r.t. $M$} is a partial example
  $\tau : F \rightarrow \{0,1\}$, where $F \subseteq \feat(M)$,
  such that $M(e)=c$ for every example $e$ that agrees with~$\tau$.
\item A \emph{global contrastive explanation (\GCEX{}) for $c$
    w.r.t. $M$} is a partial example
  $\tau : F \rightarrow \{0,1\}$, where $F \subseteq \feat(M)$,
  such that $M(e)\neq c$ for every example that agrees with~$\tau$.
\end{itemize}

For each of the above explanation types, each of the considered model
types $\MM{}$, and depending on whether or
not one wants to find a subset minimal or cardinality-wise minimum
explanation, one can now define the corresponding computational
problem. For instance:

\pbDef{\MM{}-\textsc{Subset-Minimal Local Abductive Explanation
    (\SMLAEX{})}}{A model $M \in \MM$ and an example $e$.}{Find a subset minimal local
  abductive explanation for $e$ w.r.t.~$M$.}

\pbDef{\MM{}-\textsc{Cardinality-Minimal Local Abductive Explanation
    (\MLAEX{})}}{A model $M \in \MM$, an example $e$, and an integer
  $k$.}{Is there a local explanation for $e$ w.r.t. $M$ of size at most $k$?}
The problems $\MM$-$\PPP_\SS$ and $\MM$-$\PPP_\CD$ for $\PPP\in \{ \GAEX$,
$\LCEX$, $\GCEX \}$ are defined analogously.

Finally, for these problems, we will consider natural parameters listed
in \Cref{tab:parms}; not all parameters apply to all considered
problems. We denote a problem $X$ parameterized by parameters $p,q,r$
by $X(p+q+r)$.





\section{Overview of Results}\label{sec:overview}

As we consider several problems, each with several variants and
parameters, there are thousands (indeed $4\cdot 576=2304$)
of combinations to consider. We
therefore provide a condensed summary of our results in
\Cref{fig:DTresults,fig:DSDLresults}.

The first column in each table indicates whether a result applies to
the cardinality-minimal or subset-minimal variant of the explanation
problem (i.e., to $X_\SS$ or $X_\CD$, respectively).  The next four
columns in
\Cref{fig:DTresults,fig:DSDLresults} indicate the parameterization,
the parameters are explained in \Cref{tab:parms}. A ``p'' indicates
that this parameter is part of the parameterization, a ``\N''
indicates that it isn't. A ``c'' means the parameter is set to a
constant, ``1'' means the constant is 1.

By default, each row in the tables applies to all four problems
\LAEX, \GAEX, \GCEX, and \LCEX. However, if a result only
applies to \LCEX, it is stated in parenthesis. So, for instance, the
first row of \Cref{fig:DTresults} indicates that \DT-\SMLAEX{},
\DT-\SMGAEX, \DT-\SMGCEX, and \DT-\SMLCEX, where the
ensemble consists of a single \DT, can be solved in polynomial
time. 
The penultimate row of \Cref{fig:DTresults}
indicates that \RF-\MLAEX,
\RF-\MGAEX{} and \RF-\MGCEX{} are \coNPh{} even if
$\mnlsize{}+\sizeelem{}+\xpsize{}$ is constant,
and \RF-\MLCEX{} is \Wh{1} parameterized by \xpsize{} even if 
$\mnlsize{}+\sizeelem{}$ is constant.
Finally, the $\star$ indicates a minor distinction in the complexity
between \DT-\MLAEX{} and the two problems \DT-\MGAEX{} and \DT-\MGCEX{}. That
is, if the cell contains $\NPtable{}^\star$ or $\pNPtable{}^\star$, then \DT-\MLAEX{} is
\NPh{} or \pNPh{}, respectively, and neither \DT-\MGAEX{} nor
\DT-\MGCEX{} are in \P{} unless $\FPT{}=\W{1}$.


We only state in the tables those results that are not implied by
others. Tractability results propagate in the following list from left
to right, and hardness results propagate from right to left.

\begin{eqnarray*}
  \label{eq:1}
  \text{$\CD$-minimality} &\Rightarrow& \text{$\SS$-minimality}\\[-3pt]
  \text{set $A$ of parameters} &\Rightarrow& \text{set $B\supseteq A$ of parameters}\\[-3pt]
  \text{not ensemble of models}   &\Rightarrow& \text{single model}\\[-3pt]
  \text{ordered}&\Rightarrow& \text{not ordered}\\[-3pt]           
\end{eqnarray*}
For instance, the tractability of $X_\CD$ implies the
tractability of $X_\SS$, and the hardness of $X_\SS$ implies the
hardness of $X_\CD$.\iflong\footnote{Note that even though the inclusion-wise
  minimal versions of our problems are defined in terms of finding
  (instead of decision), the implication still holds because there is
  a polynomial-time reduction from the finding version
  to the decision version of \MLAEX{}, \MLCEX{}, \MGAEX{}, and
  \MGCEX{}.}\fi

\begin{table}[tbh]
  \centering
  \begin{tabular}{@{}l@{~~}l@{}}
    \toprule
    parameter & definition \\
    \midrule
    \enssize & number of elements of the ensemble\\
    \mnlsize & largest number of $\MNL$ of any ensemble element \\
    \termselem & largest number of terms of any ensemble element \\
    \termsize & size of a largest term of any ensemble element\\
    \sizeelem & largest size of any ensemble element    \\
    \xpsize & size of the explanation\\
              \bottomrule
  \end{tabular}
  \caption{Main parameters considered. Note that some parameters (such
    as \mnlsize) only
    apply to specific model types.}\label{tab:parms} 
\end{table}

\newcommand{\Tnum}[1]{}



\begin{table}[htb]
    \Crefname{theorem}{Thm}{Thms}
\centering
\begin{tabular}{@{}cc@{}c@{}c@{}cc@{}r@{}}
  \toprule
 \rot{minimality} &\rot{\enssize}&\rot{\mnlsize} &\rot{\sizeelem} &\rot{\xpsize} &\rot{complexity} &\rot{result} \\
  \midrule 
$\SS$ &1 &\N &\N &\N &\P &\Cref{th:DT-SLGA-P}\Tnum{4,5} \\
$\CD$ &1 &\N &\N &\N &$\NPtable^{\star}$(\P) &
\ifshort\Cref{th:DT-MLAEX-W2,th:DT-MGAEX-W1,th:DT-SLGA-P}\fi\iflong\Cref{th:DT-MLAEX-W2,th:DT-MGAEX-W1,th:DT-MLCEX}\fi\Tnum{4,5} \\
$\CD$ &1 &\N &\N &p &\Wtable{1}(\P) &
\ifshort\Cref{th:DT-MLAEX-W2,th:DT-MGAEX-W1,th:DT-SLGA-P}\fi\iflong\Cref{th:DT-MLAEX-W2,th:DT-MGAEX-W1,th:DT-MLCEX}\fi\Tnum{4,5} \\
$\CD$ &1 &\N &\N &p &\XP (\P) &\ifshort\Cref{th:DT-LGA-XP,th:DT-SLGA-P}\fi\iflong\Cref{th:DT-LGA-XP,th:DT-MLCEX}\fi\Tnum{4,6} \\
$\SS$ &p &\N &\N &\N &\coWtable{1}(\Wtable{1}) &\Cref{th:RF-W1-ES}\Tnum{8} \\
$\SS$ &p &\N &\N &\N &\XP &\Cref{th:Rf-XP-e}\Tnum{10} \\
$\CD$ &p &\N &\N &\N &$\pNPtable^{\star}$(\XP) &\Cref{th:DT-MLAEX-W2,th:DT-MGAEX-W1,th:Rf-XP-e}\Tnum{10} \\
$\CD$ &p &p &\N &\N &\FPT &\Cref{cor:RF-FPT-MNL}\Tnum{10} \\
$\CD$ &p &\N &p &\N &\FPT &\Cref{cor:RF-FPT-MNL}\Tnum{7} \\
$\CD$ &p &\N &\N &c(p) &\coWtable{1}(\Wtable{1}) &\Cref{th:RF-W1-ES}\Tnum{8} \\
$\SS$ &\N &c &c &\N &\coNPtable (\NPtable) &\Cref{th:RF-paraNP}\Tnum{9} \\
$\CD$ &\N &c &c &c(p) &\coNPtable (\Wtable{1}) &\Cref{th:RF-paraNP}\Tnum{9} \\
$\CD$ &\N &\N &\N &p &\copNPtable(\XP) &\Cref{th:RF-paraNP,th:MLCEX-XP}\Tnum{1} \\
  \bottomrule
\end{tabular}
\caption{  
   Explanation complexity when the model is a \DT{}, an \DTO{}, or an
   (ordered) ensemble thereof.  
   Since the results for \DT{} and \DTO{} as well as the
   results for \RF{} and \DTOEO{} are identical,  
   we do not distinguish between them in the table.  
   See \Cref{sec:overview} for details on how to read the table.  
}\label{fig:DTresults}
\end{table}

\begin{table}[tbh]
    \Crefname{theorem}{Thm}{Thms}
  \Crefname{corollary}{Cor}{Cors}
  \centering
  \begin{tabular}{@{}cc@{}c@{}c@{}ccr@{}}
    \toprule
\rot{minimality} & \rot{\enssize}& \rot{\termselem} & \rot{\termsize}  & \rot{\xpsize} &\rot{complexity} & \rot{result} \\
 \midrule

$\SS$ &1 &\N &c & \N &\coNPtable(\NPtable) &\Cref{th:DS-paraCoNP}\Tnum{15} \\
$\CD$ &1 &\N &\N &p &\copNPtable(\Wtable{2}) & \Cref{th:DS-paraCoNP,th:DS-SMLCEX-W1}\Tnum{31} \\
$\CD$ &c &\N &p &p  &\copNPtable(\FPT) & \Cref{th:DS-paraCoNP} and \Cref{cor:DS-MLCEX-FPT-const-ens}\Tnum{33} \\
$\CD$ &p &p &\N &\N &\FPT &\Cref{cor:ds-ensa}\Tnum{12} \\
$\CD$ &p &\N &c &p  &\copNPtable(\Wtable{1}) & \Cref{th:DS-paraCoNP,th:DSE-SMLCEX-W1}\Tnum{33} \\
$\SS$ &\N &c &c &\N &\coNPtable(\NPtable) &\Cref{th:DSE-paraNP}\Tnum{17} \\
$\CD$ &\N &c &c &c(p) &\coNPtable(\Wtable{1}) &\Cref{th:DSE-paraNP}\Tnum{17} \\
$\CD$ &\N &\N &\N &p &\copNPtable(\XP) &\Cref{th:DS-paraCoNP,th:MLCEX-XP}\Tnum{1} \\
    \bottomrule
  \end{tabular}
  \caption{
   Explanation complexity when the model is a  \DS, \DL,  or an ensemble thereof.  
   See \Cref{sec:overview} for details on how to read the table. 
   }\label{fig:DSDLresults}
\end{table}

\section{Algorithmic Results}\label{sec:algorithmic}

\newcommand{\val}[3]{\textsf{val}(#1,#2,#3)}
\newcommand{\IG}{\textsf{IG}}
\newcommand{\G}{\textsf{G}}
\newcommand{\OUT}{\textsf{O}}

In this section, we will present our algorithmic results.
We start with some general observations that are independent of a
particular model type.
\begin{theorem}\label{th:MLCEX-XP}
  Let $\MM{}$ be any model type such that $M(e)$ can be computed in
  polynomial-time for every $M \in \MM$. \MM{}-\MLCEX{} parameterized
  by \xpsize{} is in \XP{}.
\end{theorem}
\begin{proof}
  Let $(M,e,k)$ be the given instance of \MM{}-\MLCEX{} and suppose that
  $A \subseteq \feat(M)$ is a cardinality-wise minimal local contrastive
  explanation for $e$ w.r.t. $M$. Because $A$ is cardinality-wise minimal,
  it holds the example $e_A$ obtained from $e$ by setting
  $e_A(f)=1-e(f)$ for every $f \in A$ and $e_A(f)=e(f)$ otherwise, is
  classified differently from $e$, i.e., $M(e)\neq M(e_A)$. 
  Therefore, a set $A \subseteq \feat(M)$ is a 
  cardinality-wise minimal local contrastive explanation for $e$
  w.r.t. $M$ if and only if $M(e)\neq M(e_A)$ and there is no
  cardinality-wise smaller set $A'$ for which this is the case.
  This
  now allows us to obtain an \XP{} algorithm for $\MM{}$-\MLCEX{} as
  follows. We first enumerate all possible subsets $A \subseteq
  \feat(M)$ of size at most $k$ in time $\bigoh(|\feat(M)|^k)$ and for
  each such subset $A$ we test in polynomial-time 
  if $M(e_A)\neq M(e)$. If so, we output that $(M,e,k)$ is a yes-instance
  and if this is not the case for any of the enumerated subsets, we
  output correctly that $(M,e,k)$ is a no-instance.
\end{proof}

The remainder of the section is organised as follows. First in
\Cref{ssec:bc}, we provide a very general result about Boolean circuits,
which will allow us to show a variety of algorithmic results
for our models. We then provide our algorithms for the considered
models in Subsections~\ref{ssec:algDT} to~\ref{ssec:algDSDL}.

\subsection{A Meta-Theorem for Boolean Circuits}\label{ssec:bc}

Here, we present our algorithmic result for Boolean circuits that are
allowed to employ majority circuits. In particular, we will show
that all considered explanation problems are fixed-parameter tractable
parameterized by the so-called 
rank-width of the Boolean circuit as
long as the Boolean circuit uses only a constant number of majority
gates\ifshort; see, e.g.,~\cite{DBLP:journals/jct/OumS06} for a
  definition of rank-width\fi. Since our considered models can be naturally translated
into Boolean circuits, which require majority gates in the case of
ensembles, we will obtain a rather large number of algorithmic
consequences from this result by providing suitable reductions of our
models to Boolean circuits in the following subsections.

\iflong
We start by introducing Boolean circuits.
A \emph{Boolean circuit} (\BC{})  is a directed acyclic graph $D$ with a unique
sink vertex $o$ (output gate) such that every vertex $v \in V(D)
\setminus \{o\}$ is either:
\begin{itemize}
\item an \emph{IN-gate} (input gate) with no incoming arcs,
\item an \emph{AND-gate} with at least one incoming arc,
\item an \emph{OR-gate} with at least one incoming arc,
\item a \emph{MAJ-gate} (majority gate) with at least one incoming arc and an integer
  threshold $t_v$, or
\item a \emph{NOT-gate} with exactly one incoming arc.
\end{itemize}
We denote by $\IG(D)$ the set of all input gates of $D$ and by
$\MAJ(D)$ the set of all MAJ-gates of $D$.
For an assignment $\alpha : \IG(D) \rightarrow
\{0,1\}$ and a vertex $v \in V(D)$, we denote by $\val{v}{D}{\alpha}$ the
value of the gate $v$ after assigning all input gates according to
$\alpha$.
That is, $\val{v}{D}{\alpha}$ is recursively defined as follows: If
$v$ is an input gate, then $\val{v}{D}{\alpha}=\alpha(v)$, if $v$ is an
AND-gate (OR-gate), 
then $\val{v}{D}{\alpha}=\bigwedge_{n \in
  N_D^-(v)}\val{n}{D}{\alpha}$ ($\val{v}{D}{\alpha}=\bigvee_{n \in
  N_D^-(v)}\val{n}{D}{\alpha}$), and if $v$ is a MAJ-gate, then
$\val{v}{D}{\alpha}=(t_v\leq |\SB n \SM n \in N^-_D(v) \land
\val{n}{D}{\alpha}=1\SE|)$.
Recall that
$N_D^-(v)$ denotes the set of all incoming neighbours of $v$ in $D$.
We set
$\OUT(D,\alpha)=\val{o}{D}{\alpha}$. We say that $D$ is a $c$-\BC{} if
$c$ is an integer and $D$ contains at most $c$ MAJ-gates.
\fi

\newcommand{\MSO}{\textsf{MSO}$_1$}
\newcommand{\MSOE}{\textsf{MSOE}$_1$}


\iflong
We consider \emph{Monadic Second Order} (\MSO{}) logic on structures
representing \BC{}s as a directed acyclic graph with unary relations
to represent the types of gates. That is the structure associated with
a given \BC{} $D$ has $V(D)$ as its universe and contains the
following unary and binary relations over $V(D)$:
\begin{itemize}
\item the unary relations $\textup{NOT}$, $\textup{AND}$, $\textup{OR}$, $\textup{MAJ}$, and $\textup{IN}$
  containing all
  NOT-gates, all AND-gates, all OR-gates, all MAJ-gates, and all
  IN-gates of $D$, respectively,
\item the binary relation $E$ containing all pairs $x,y \in V(D)$
  such that $(x,y) \in E(D)$.
\end{itemize}

We assume an infinite supply of \emph{individual variables} and
\emph{set variables}, which we denote by lower case and upper case
letters, respectively. The available \emph{atomic formulas} are
$P g$ (``the value assigned to variable $g$ is contained in the unary
relation or set variable $P$''),
$E (x,y)$ 
(``vertex $x$ is the tail of an edge with head $y$''),
$x=y$ (equality), and $x\neq y$
(inequality). \emph{\MSO{} formulas} are built up
from atomic formulas using the usual Boolean connectives
$(\lnot,\land,\lor,\rightarrow,\leftrightarrow)$, quantification over
individual variables ($\forall x$, $\exists x$), and quantification over
set variables ($\forall X$, $\exists X$). 

In order to be able to deal with MAJ-gates, we will need a slightly
extended version of \MSO{} logic, which we denote by \MSOE{} and which in turn is a slightly
restricted version of the so-called distance neighbourhood logic that
was introduced by \citet{BergougnouxDJ23}. \MSOE{} extends \MSO{} with
\emph{set terms}, which are built from
set-variables, unary relations, or other set terms by
applying standard set operations such as intersection ($\cap$), union
($\cup$), subtraction ($\setminus$), or complementation (denoted by a
bar on top of the term). Note that the distance neighbourhood logic
introduces neighbourhood terms, which extend set terms by allowing an
additional neighbourhood operator on sets that is not required for our
purposes. Like distance neighbourhood logic, \MSOE{} also allows for comparisons between set terms, i.e.,
we can write $t_1=t_2$ or $t_1\subseteq t_2$ to express that the set
represented by the set term $t_1$ is equal or a subset of the set represented by
the set term $t_2$, respectively.
Most importantly for modelling MAJ-gates is that \MSOE{} allows for \emph{size measurement
  of terms}, i.e., for a set term $t$ and an integer $m$, we can write
$m\leq |t|$ to express that the set represented by $t$ contains at
least $m$ elements.

Let $\Phi$ be an \MSOE{} formula (sentence). For
a \BC{} $D$ and possibly an additional unary relation $U \subseteq
V(D)$, we write $(D,U) \models \Phi$ if $\Phi$ holds
true on the structure representing $D$ with additional unary relation $U$. The following proposition
is crucial for our algorithms based on \MSOE{} and essentially
provides an efficient algorithm for a simple optimisation variant of
the model checking problem for \MSOE{}.
\begin{proposition}[{\cite[Theorem 1.2]{BergougnouxDJ23}}]\label{pro:MSOE}
  Let $D$ be a $c$-\BC{}, $U \subseteq V(D)$, and let $\Phi(S_1,\dotsc,S_\ell)$ be an \MSOE{} formula with free (non-quantified) set
  variables $S_1,\dotsc,S_\ell$. The problem of computing sets
  $B_1,\dotsc,B_\ell \subseteq V(D)$ such that $(D,U)\models
  \Phi(B_1,\dotsc,B_\ell)$ and $\sum_{i=1}^\ell|B_i|$ is minimum
  is fixed-parameter tractable parameterized by $\rw(D)+|\Phi(S_1,\dotsc,S_\ell)|$.
\end{proposition}
\fi

\ifshort
  To show our algorithmic result for Boolean circuits given below, we
  make use of an only recently developed meta-theorem~\cite[Theorem
  1.2]{BergougnouxDJ23} involving an extension of Monadic second order
  logic that allows us to easily model majority gates of Boolean circuits.
\fi
\begin{theorem}
  \label{the:solve-circ}
  $c$-\BC{}-\MLAEX{}, $c$-\BC{}-\MGAEX{}, $c$-\BC{}-\MLCEX{},
  $c$-\BC{}-\MGCEX{} are fixed-parameter tractable parameterized by
  the rank-width of the circuit.
\end{theorem}
\iflong
  \begin{proof}
  Let $D$ be a $c$-\BC{} with output gate $o$. We will define one
  \MSOE{} formula for each of the four considered problems. That is, we will define
  the formulas $\Phi_{\textsf{LA}}(S)$, $\Phi_{\textsf{LC}}(S)$,
  $\Phi_{\textsf{GA}}(S_0,S_1)$, and $\Phi_{\textsf{GC}}(S_0,S_1)$ such that:
  \begin{itemize}
  \item $(D,T) \models \Phi_{\textsf{LA}}(S)$ if and only if $S$
    is a local abductive explanation for $e$ w.r.t. $D$. Here, $e$ is
    the given example and $T$ is a unary relation on $V(D)$ given as
    $T=\SB v \in \IG(D) \SM e(v)=1 \SE$.
  \item $(D,T) \models \Phi_{\textsf{LC}}(S)$ if and only if $S$
    is a local contrastive explanation for $e$ w.r.t. $D$. Here, $e$
    is the given example and $T$ is a unary relation on $V(D)$ given
    as $T=\SB v \in \IG(D) \SM e(v)=1 \SE$.
  \item $(D,C) \models \Phi_{\textsf{GA}}(S_0,S_1)$ if and only if the
    partial assignment $\tau : S_0\cup S_1 \rightarrow \{0,1\}$ with
    $\tau(s)=0$ if $s \in S_0$ and $\tau(s)=1$ if $s \in S_1$
    is a global abductive explanation for $c$ w.r.t. $D$.
    Here, $c \in \{0,1\}$ is the given class and $C$ is a unary
    relation on $V(D)$ that is empty if $c=0$ and otherwise contains
    only the output gate $o$.
  \item $(D,C) \models \Phi_{\textsf{GC}}(S_0,S_1)$ if and only if
    the
    partial assignment $\tau : S_0\cup S_1 \rightarrow \{0,1\}$ with
    $\tau(s)=0$ if $s \in S_0$ and $\tau(s)=1$ if $s \in S_1$
    is a global contrastive explanation for $c$ w.r.t. $D$.
    Here, $c \in \{0,1\}$ is the given class and $C$ is a unary
    relation on $V(D)$ that is empty if $c=0$ and otherwise contains
    only the output gate $o$.
  \end{itemize}
  Since, 
  each of the formulas will have constant length, the theorem
  then follows immediately from \Cref{pro:MSOE}.

  We start by defining the auxiliary formula $\textsf{CON}(A)$ such that $D
  \models \textsf{CON}(B)$ if and only if $B=\SB v \SM v \in V(D)
  \land \val{v}{D}{\alpha_B}=1\SE$, where $\alpha_B : \IG(D)
  \rightarrow \{0,1\}$ is the assignment of the input gates of $D$
  defined by setting $\alpha_B(v)=1$ if $v \in B$ and $\alpha_B(v)=0$,
  otherwise. 
  In other words, 
  $D \models \textsf{CON}(B)$ holds 
  if and only if 
  $B$ represents a consistent
  assignment of the gates, i.e., exactly those gates are in $B$ that
  are set to $1$ if the circuit is evaluated for the input assignment $\alpha_B$.
  The formula $\textsf{CON}(A)$ is defined as
  $\textsf{CON}_\MAJ(A) \land \textsf{CON}'(A)$, where:

  \[
    \textsf{CON}_\MAJ(A) = 
    \bigwedge_{g \in \MAJ(D)} 
    \begin{array}{ll}
    g \in A \leftrightarrow (\exists N. 
    |N| \geq t_g \land\\ 
     (\forall n. n \in N \leftrightarrow (E(n,g) \land n \in A )))
    \end{array}
  \]
  and
  \[ \textsf{CON'}(A) = \forall g.  \begin{array}{lll}
    &(\textup{AND}(g) &\rightarrow (g \in A \leftrightarrow \forall n. E(n,g) \rightarrow n \in A)) \land\\
    & (\textup{OR}(g) &\rightarrow (g \in A \leftrightarrow \exists n. E(n,g) \land n \in A)) \land\\
    & (\textup{NOT}(g) &\rightarrow (g \in A \leftrightarrow \exists n. E(n,g) \land n \notin A))\\
  \end{array}
  \]

  We also define the formula $\textsf{EXTEND}(A, A_0, A_1)$,  
  which ensures that $D \models \textsf{EXTEND}(B, B_0, B_1)$  
  if and only if $B$ represents a consistent gate assignment,  
  and for each $i \in \{0,1\}$, all gates in $B_i$ are set to $i$.  
  In other words, the partial assignment defined by $B_0$ and $B_1$  
  is extended to a total, consistent assignment represented by $B$.

    \[ \textsf{EXTEND}(A,A_0,A_1) = \textsf{CON}(A)
  \land (A_0\cap A=\emptyset) \land (A_1\subseteq A) 
  \]
  

  We are now ready to define the formula $\textsf{GLOBAL}(S_0, S_1, x)$,  
  which expresses that  
  $D \models \textsf{GLOBAL}(B_0, B_1, y)$ if and only if  
  $B_0, B_1 \subseteq \IG(D)$, $B_0 \cap B_1 = \emptyset$, and  
  for every consistent assignment $B$ extending $B_0$ and $B_1$,  
  we have that $o \in B$ if and only if $y$ holds.  
  In other words, every input assignment $\alpha$  
  that agrees with the partial assignment represented by $B_0$ and $B_1$  
  is classified by the circuit $D$ as $y$.

    \[
    \textsf{GLOBAL}(S_0,S_1,x) = 
    \begin{array}{ll}
      & (\forall g. g \in S_0\cup S_1 \rightarrow \textup{IN}(g))  \land S_1 \cap S_2=\emptyset \land \\
      & (\forall A. \textsf{EXTEND}(A,S_0,S_1) \rightarrow (x \leftrightarrow o \in A))
    \end{array}
  \]

  This is sufficient to define both types of global explanations.  
  Note that the only difference between them lies in how they use  
  the unary relation $C$ to determine the target classification.
  \[
    \begin{array}{ll}
      \Phi_{\textsf{GA}}(S_0,S_1) = & \textsf{GLOBAL}(S_0,S_1, C o)
    \end{array}
  \]
    \[
    \begin{array}{ll}
      \Phi_{\textsf{GC}}(S_0,S_1) = & \textsf{GLOBAL}(S_0,S_1, \neg C o)
    \end{array}
  \]

  To define local explanations, we need the additional 
  formula $\textsf{EXAMPLE}(I_0,I_1)$ such that
  $T \models \textsf{EXAMPLE}(B_0,B_1)$ if and only if
  $B_0 = \IG(D)\setminus T$ and $B_1 = T$.
  In other words, $B_0$ and $B_1$ represent the input assignment $T$  
  by explicitly partitioning the input gates into those assigned $0$ and $1$, respectively.

    \[ \textsf{EXAMPLE}(I_0,I_1) = 
    (\forall g. T(g) \leftrightarrow g \in I_1)
    \land
    (\forall g. \textup{IN}(g) \rightarrow (g \in I_1 \leftrightarrow g \notin I_0))
    \]

    We are ready to define $\Phi_{\textsf{LA}}(S)$.
    
  \[
  \Phi_{\textsf{LA}}(S) = 
    \begin{array}{ll}
      & \exists I_0. \exists I_1.\exists A.  \textsf{EXAMPLE}(I_0,I_1) \land \textsf{EXTEND}(A,I_0,I_1) \land\\
      & \textsf{GLOBAL}(I_0 \cap S, I_1 \cap S, o \in A)
      \end{array}
  \]
  Note that, due to $\textsf{EXAMPLE}$, the sets $I_0$ and $I_1$  
  encode the input assignment $T$.  
  By $\textsf{EXTEND}$, the set $A$ represents a consistent gate assignment  
  for the example $T$. Therefore, the condition $o \in A$  
  captures whether the circuit $D$ classifies $T$ as $1$.
  The formula $\textsf{GLOBAL}$ then verifies whether all examples  
  that agree with the partial assignment of input
  obtained from $T$  
  by restricting to the features in $S$  
  are classified in the same way as $T$.

    Finally, we define $\Phi_{\textsf{LC}}(S)$.
  \[
  \Phi_{\textsf{LC}}(S) = 
    \begin{array}{ll}
      & \exists I_0. \exists I_1.\exists A.  \textsf{EXAMPLE}(I_0,I_1) \land \textsf{EXTEND}(A,I_0,I_1) \land\\
      & \exists A'. \textsf{EXTEND}(A',(I_0 \setminus S) \cup (I_1 \cap S), (I_1    \setminus S) \cup (I_0 \cap S)) \land \\
      & (o \in A \leftrightarrow o \notin A')
      \end{array}
  \]
  As before, the sets $I_0$ and $I_1$ encode the input assignment $T$,  
  and the set $A$ represents a consistent gate assignment w.r.t. $I_0$ and $I_1$.
  Additionally, the set $A'$ corresponds to a consistent gate assignment  
  for the example obtained from $T$ by flipping the values assigned  
  to the features in $S \cap \IG(D)$.  
  The final condition checks whether these two examples  
  (i.e., $T$ and its $S$-flipped variant)  
  are classified differently by the circuit.

  Note that $\Phi_{\textsf{LC}}(S)$ is defined without using $\textsf{GLOBAL}$,  
  which highlights a structural distinction between local contrastive explanations  
  and the other types. Indeed, this difference manifests itself throughout the paper,  
  particularly when we observe distinct complexity classes  
  emerging for the corresponding decision problems.
\end{proof}
\fi

\subsection{DTs and their Ensembles}\label{ssec:algDT}

Here, we present our algorithms for \DT{}s and their ensembles.
\iflong
We
start with a simple translation from \DT{}s to \BC{}s that allows us to
employ \Cref{the:solve-circ} for \DT{}s.
\begin{lemma}\label{lem:dt-trans-circ}
  There is a polynomial-time algorithm that given a \DT{} $\TTT=(T,\lambda)$ and a
  class $c$ produces a circuit $\CIRC(\TTT,c)$ such that:
  \begin{enumerate}[(1)]
  \item for every example $e$, it holds that $\TTT(e)=c$ if and only if
    (the assignment represented by) $e$ satisfies $\CIRC(\TTT,c)$ and
  \item $\rw(\CIRC(\TTT,c)) \leq 3\cdot 2^{|\MNL(\TTT)|}$. 
  \end{enumerate}
\end{lemma}
\iflong\begin{proof}
  Let $\TTT=(T,\lambda)$ be the given \DT{} and suppose that
  $\MNL(\TTT)$ is equal to the number of negative leaves;
  the construction of the circuit $\CIRC(\TTT,c)$ is analogous if instead
  $\MNL(\TTT)$ is equal to the number of positive leaves.
  We first construct the circuit $D$ such that $D$ is satisfied by $e$
  if and only if $\TTT(e)=0$. $D$ contains one input gate $g_f$
  and one NOT-gate $\overline{g_f}$, whose only incoming arc is from $g_f$,
  for every feature in $\feat(\TTT)$. Moreover, for every $l \in
  L_0(\TTT)$, $D$ contains an AND-gate $g_l$, whose incoming
  arcs correspond to the partial assignment $\alpha_{\TTT}^l$,
  i.e., for every feature $f$ assigned by $\alpha_{\TTT}^l$,
  $g_l$ has an incoming arc from $g_f$ if
  $\alpha_{\TTT}^l(f)=1$ and an incoming arc from
  $\overline{g_f}$ otherwise. Finally, $D$ contains the OR-gate $o$,
  which also serves as the output gate of $D$, that has one incoming
  arc from $g_l$ for every $l \in L_0$. This completes the
  construction of $D$ and it is straightforward to show that $D$ is
  satisfied by an example $e$ if and only if $\TTT(e)=0$.
  Moreover, using \Cref{lem:ranktree}, we obtain that $D$ has treewidth at most $|\MNL(\TTT)|+1$ because
  the graph obtained from $D$ after removing all gates $g_l$ for every
  $l \in L_0(\TTT)$ is a tree and therefore has treewidth at
  most $1$. Therefore, using \Cref{lem:ranktree}, we obtain that $D$ has
  rank-width at most $3\cdot 2^{|\MNL(\TTT)|}$. Finally, $\CIRC(\TTT,c)$ can now be obtained from
  $D$ as follows. If $c=0$, then $\CIRC(\TTT,c)=D$. Otherwise,
  $\CIRC(\TTT,c)$ is obtained from $D$ after adding one OR-gate
  that also serves as the new output gate of $\CIRC(\TTT,c)$
  and that has only one incoming arc from $o$.
\end{proof}\fi

We now provide a translation from \RF{}s to $1$-\BC{}s that will 
allow us to obtain tractability results for \RF{}s.

\begin{lemma}\label{lem:rf-trans-circ}
  There is a polynomial-time algorithm that given a \RF{} $\mathcal{F}$ and a
  class $c$ produces a circuit $\CIRC(\mathcal{F},c)$ such that:
  \begin{enumerate}[(1)]
  \item for every example $e$, it holds that $\mathcal{F}(e)=c$ if and only if
    (the assignment represented by) $e$ satisfies $\CIRC(\mathcal{F},c)$ and
  \item $\rw(\CIRC(\mathcal{F},c)) \leq 3\cdot 2^{\sum_{\TTT \in \mathcal{F}}|\MNL(\TTT)|}$
  \end{enumerate}
\end{lemma}
\iflong\begin{proof}
  We obtain the circuit $\CIRC(\mathcal{F},c)$ from the (not
  necessarily disjoint) union of the circuits $\CIRC(\TTT,c)$
  for every $\TTT \in \mathcal{F}$, which we introduced in
  \Cref{lem:dt-trans-circ},  after adding a new MAJ-gate
  with threshold $\lfloor|\mathcal{F}|/2\rfloor+1$,
  which also serves as the output gate of $\CIRC(\mathcal{F},c)$, that
  has one incoming arc from the output gate of $\CIRC(\TTT,c)$
  for every $\TTT \in \mathcal{F}$. Clearly,
  $\CIRC(\mathcal{F},c)$ satisfies (1). Moreover, to see that it also satisfies
  (2), recall that every circuit $\CIRC(\TTT,c)$ has only
  $\MNL(\TTT)$ gates apart from the input gates, the NOT-gates
  connected to the input gates, and the output gate. Therefore, after
  removing $\MNL(\TTT)$ gates from every circuit
  $\CIRC(\TTT,c)$ inside $\CIRC(\mathcal{F},c)$, the remaining
  circuit is a tree, which together with \Cref{lem:ranktree}
  implies (2).
\end{proof}\fi
The following result now follows immediately from
\Cref{the:solve-circ} together with \Cref{lem:rf-trans-circ}.
\fi
\ifshort
  With the help of our meta-theorem (\Cref{the:solve-circ}) together with natural translations of \DT{}s and
  \RF{}s into \BC{}s and $1$-\BC{}s, respectively, we obtain the
  following two theorems, showing that all problems are
  fixed-parameter tractable parameterized by \enssize{} plus
  \mnlsize{}.
\fi

\begin{corollary}\label{cor:RF-FPT-MNL}
  Let $\PP \in \{\LAEX, \LCEX, \GAEX, \GCEX\}$. 
  \RF{}-\MPP{}$(\enssize{}+\mnlsize{})$ and therefore also
  \RF{}-\MPP{}$(\enssize{}+\sizeelem{})$ is
  in \FPT{}. 
\end{corollary}



\iflong
We now give our polynomial-time algorithms for \DT{}s. We start with
the following known result for contrastive explanations. 
\begin{theorem}[{~\cite[Lemma 14]{BarceloM0S20}}]\label{th:DT-MLCEX}
  There is a polynomial-time algorithm that given a \DT{} $\TTT$ and
  an example $e$ outputs a (cardinality-wise) minimum local
  contrastive explanation for $e$ w.r.t. $\TTT$ or no if such an
  explanation does not exist. Therefore, 
  \DT{}-\MLCEX{} can be solved in polynomial-time. 
\end{theorem}
\fi

The following auxiliary lemma provides polynomial-time algorithms for
testing whether a given subset of features $A$ or partial
example $e'$ is a local abductive, global abductive, or global
contrastive explanation for a given example $e$ or class $c$ w.r.t. a
given \DT{} $\TTT$.
\begin{lemma}\label{lem:dt-sm-testsol}
  Let $\TTT$ be a \DT{}, let $e$ be an example and let $c$ be a class.
  There are polynomial-time algorithms for the following problems:
  \begin{enumerate}[(1)]
  \item Decide whether a given subset $A \subseteq \feat(\TTT)$ of
    features is a local abductive explanation for $e$ w.r.t. $\TTT$.
  \item Decide whether a given partial example $e'$ is a global
    abductive/contrastive explanation for $c$ w.r.t. $\TTT$.
  \end{enumerate}
\end{lemma}
\iflong\begin{proof}\fi
\ifshort\begin{proof}[Proof Sketch]\fi  
  Let $\TTT$ be a \DT{}, let $e$ be an example and let $c$ be a class.
  Note that we assume here that $\TTT$ does not have any contradictory
  path\ifshort.\fi\iflong, i.e., a root-to-leaf path that contains that assigns any
  feature more than once. Because if this was not the case, we could
  easily simplify $\TTT$ in polynomial-time. \fi

  We start by showing (1).
  A subset $A \subseteq \feat(\TTT)$ of
  features is a
  local abductive explanation for $e$ w.r.t. $\TTT$ if and only if the
  \DT{} $\TTT_{|e_{|A}}$ does only contain
  $\TTT(e)$-leaves, which can clearly be decided in polynomial-time.
  Here, $e_{|A}$ is the partial example equal to the
  restriction of $e$ to $A$. Moreover, $\TTT_{|e'}$ for a partial
  example $e'$ is the \DT{} obtained from $\TTT$ after removing every
  $1-e'(f)$-child from every node $t$ of $\TTT$ assigned to a feature
  $f$ for which $e'$ is defined. \ifshort The proof for (2) is similar.\fi
  \iflong
    
  Similarly, for showing (2), observe that partial example (assignment) $\tau : F \rightarrow
  \{0,1\}$ is a global abductive explanation for $c$ w.r.t. $\TTT$ if and only if the \DT{}
  $\TTT_{|\tau}$ does only contain
  $c$-leaves, which can clearly be decided in polynomial-time.

  Finally, note that a partial example $\tau : F \rightarrow
  \{0,1\}$ is a global contrastive explanation for $c$ w.r.t. $\TTT$ if and only if the \DT{}
  $\TTT_{|\tau}$ does not contain any
  $c$-leaf, which can clearly be decided in polynomial-time.\fi
\end{proof}

Using dedicated algorithms for the inclusion-wise minimal variants of
\LAEX{}, \GAEX{}, \iflong and \fi \GCEX{} \iflong together with
  \Cref{th:DT-MLCEX}\fi\ifshort and using the
  polynomial-time algorithm for the
  cardinality-wise minimal version of \LCEX{} given in~\cite[Lemma
  14]{BarceloM0S20}\fi, we obtain the following result.
\begin{theorem}
  \label{th:DT-SLGA-P}
    Let $\PP \in \{\LAEX, \LCEX, \GAEX, \GCEX\}$. 
  \DT{}-\SPP{} \ifshort and \DT{}-\MLCEX{} \fi can be solved in
  polynomial-time. 
\end{theorem}
\iflong\begin{proof}\fi
\ifshort\begin{proof}[Proof Sketch]\fi  
  Note that the statement of the theorem for \DT{}-\SMLCEX{} follows
  immediately from\iflong~\Cref{th:DT-MLCEX}\fi\ifshort~\cite[Lemma 14]{BarceloM0S20}\fi. Therefore, it suffices to show
  the statement of the theorem for the remaining 3 problems.
  \iflong
  
  Let $(\TTT,e)$ be an instance of  \DT{}-\SMLAEX{}. We start by setting $A=\feat(\TTT)$. Using \Cref{lem:dt-sm-testsol}, we
  then test for any feature $f$ in $A$, whether $A\setminus \{f\}$
  is still a local abductive explanation for $e$ w.r.t. $\TTT$ in polynomial-time. If so,
  we repeat the process after setting $A$ to $A\setminus\{f\}$ and
  otherwise we do the same test for the next feature $f \in
  A$. Finally, if $A\setminus\{f\}$
  is not a local abductive explanation for every $f \in A$, then $A$ is an inclusion-wise
  minimal local abductive explanation and we can output $A$.
  \fi
  
  The polynomial-time algorithm for an instance $(\TTT,c)$ of
  \DT{}-\SMGAEX{} \iflong now\fi works as
  follows. Let $l$ be a $c$-leaf of $\TTT$; if no such $c$-leaf
  exists, then we can correctly output that there is no global
  abductive explanation for $c$ w.r.t. $\TTT$. Then, $\alpha_\TTT^l$
  is a global abductive explanation for $c$ w.r.t. $\TTT$. To obtain
  an inclusion-wise minimal solution, we do the following. Let $F=\feat(\alpha_\TTT^l)$ be
  the set of features on which $\alpha_\TTT^l$ is defined. We now test for
  every feature $f \in F$ whether the restriction
  $\alpha_\TTT^l[F\setminus \{f\}]$ of $\alpha_\TTT^l$ to $F\setminus
  \{f\}$ is a global abductive explanation for $c$ w.r.t. $\TTT$. This
  can clearly be achieved in polynomial-time with the help of
  \Cref{lem:dt-sm-testsol}. If this is true for any feature $f \in F$,
  then we repeat the process for $\alpha_\TTT^l[F\setminus \{f\}]$,
  otherwise we output $\alpha_\TTT^l$. \ifshort Very similar
    algorithms now also work for \DT{}-\SMGCEX{} and \DT{}-\SMLAEX{}.\fi \iflong A very similar algorithm now
  works for the \DT{}-\SMGCEX{} problem, i.e., instead of starting with a $c$-leaf
  $l$ we start with a $c'$-leaf, where $c \neq c'$.\fi
\end{proof}

The next theorem uses our result that the considered problems are in
polynomial-time for \DT{}s \ifshort(\Cref{th:FBDD-SLGA-P}) \fi
together with an \XP{}-algorithm that
transforms any \RF{} into an equivalent \DT{}.
\begin{theorem}\label{th:Rf-XP-e}
  Let $\PP \in \{\LAEX, \LCEX, \GAEX, \GCEX\}$. 
  \RF-\SPP{}$(\enssize{})$ and \RF-\MLCEX{}$(\enssize{})$ are in \XP{}.
\end{theorem}
\iflong\begin{proof}
  Let $\FFF=\{\TTT_1,\dotsc,\TTT_\ell\}$ be \RF{} given as
  an input to any of the five problems stated above. We start by constructing a
  \DT{} $\TTT$ that is not too large (of size at most $m^{|\FFF|}$,
  where $m=(\max \SB |L(\TTT_i)|\SM 1 \leq i \leq \ell\SE)$) and that is equivalent to $\mathcal{F}$
  in the sense that $\mathcal{F}(e)=\TTT(e)$ for every
  example in time at most $\bigoh(m^{|\FFF|})$. Since all five
  problems can be solved in polynomial-time on $\TTT$ (because of
  \Cref{th:DT-MLCEX,th:DT-SLGA-P}) this completes the
  proof of the theorem.

  We construct
  $\TTT$ iteratively as follows. First, we set
  $\TTT_1'=\TTT_1$. Moreover, for every $i>1$,
  $\TTT_i'$ is obtained from $\TTT_{i-1}'$ and
  $\TTT_i$ by making a copy $\TTT_i^l$ of $\TTT_i$ for every leaf $l$
  of $\TTT_{i-1}'$ and by identifying the root of
  $\TTT_i^l$ with the leaf $l$ of $\TTT_{i-1}'$. Then,
  $\TTT$ is obtained from $\TTT_{\ell}'$ changing the
  label of every leaf $l$ be a leaf of $\TTT_{\ell}'$ as
  follows. Let $P$ be the path
  from the root of $\TTT_{\ell}'$ to $l$ in
  $\TTT_{\ell}'$. By construction, $P$ goes through one copy of
  $\TTT_i$ for every $i$ with $1 \leq i \leq \ell$ and
  therefore also goes through exactly one leaf of every
  $\TTT_i$. We now label $l$ according to the majority of the
  labels of these leaves. This completes the construction of
  $\TTT$ 
  and it is easy to see that
  $\mathcal{F}(e)=\TTT(e)$ for every example $e$. Moreover, the
  size of $\TTT$ is at most
  $\Pi_{i=1}^\ell|L(\TTT_i)|\leq m^\ell$, where $m=(\max \SB |L(\TTT_i)|\SM
  1 \leq i \leq \ell\SE)$ and $\TTT$ can be constructed in time $\bigoh(m^\ell)$.
\end{proof}\fi

The following theorem uses an exhaustive enumeration of all possible
explanations together with \Cref{lem:dt-sm-testsol} to check whether a
set of features or a partial example is an explanation.
\begin{theorem}\label{th:DT-LGA-XP}
  Let $\PP \in \{\LAEX, \GAEX, \GCEX\}$. 
  \DT{}-\MPP{}$(\xpsize{})$ is in \XP{}.
\end{theorem}
\begin{proof}
  We start by showing the statement of the theorem for \DT{}-\MLAEX{}.
  Let $(\TTT,e,k)$ be
  an instance of \DT{}-\MLAEX{}. We first enumerate all subsets $A
  \subseteq \feat(\TTT)$ of size at most $k$ in time
  $\bigoh(|\feat(\TTT)|^k)$. For every such subset $A$, we then test
  whether $A$ is a local abductive explanation for $e$ w.r.t. $\TTT$
  in polynomial-time with the help of \Cref{lem:dt-sm-testsol}. If so,
  we output $A$ as the solution. Otherwise, i.e., if no such subset is
  a local abductive explanation for $e$ w.r.t. $\TTT$, we output
  correctly that $(\TTT,e,k)$ has no solution.

  Let $(\TTT,c,k)$ be
  an instance of \DT{}-\MGAEX{}. We first enumerate all subsets $A
  \subseteq \feat(\TTT)$ of size at most $k$ in time
  $\bigoh(|\feat(\TTT)|^k)$. For every such subset $A$, we then
  enumerate all of the at most $2^{|A|}\leq 2^k$ partial examples
  (assignments) $\tau : A \rightarrow\{0,1\}$ in time $\bigoh(2^k)$.
  For every such partial example $\tau$, we then use
  \Cref{lem:dt-sm-testsol} to test whether $\tau$ is a global
  abductive explanation for $c$ w.r.t. $\TTT$ in polynomial-time. If so,
  we output $e$ as the solution. Otherwise, i.e., if no such
  partial example is
  a global abductive explanation for $c$ w.r.t. $\TTT$, we output
  correctly that $(\TTT,c,k)$ has no solution. The total runtime of
  the algorithm is at most $2^k|\feat(\TTT)|^k|\TTT|^{\bigoh(1)}$.

  The algorithm for \DT{}-\MGCEX{} is now very similar to the above
  algorithm for \DT{}-\MGAEX{}.
\end{proof}

\subsection{DSs, DLs and their Ensembles}\label{ssec:algDSDL}

This subsection is devoted to our algorithmic results for \DS{},
\DL{}s and their ensembles. Our first algorithmic result is again
based on our meta-theorem (\Cref{the:solve-circ}) and a suitable
translation from \DSE{} and \DLE{} to a Boolean circuit.

\iflong
\begin{lemma}\label{lem:dsl-trans-circ}
  Let $\MM \in \{\DS,\DL\}$.
  There is a polynomial-time algorithm that given an $\MM$ $L$ and a
  class $c$ produces a circuit $\CIRC(L,c)$ such that:
  \begin{itemize}
  \item for every example $e$, it holds that $L(e)=c$ if and only if $e$ satisfies $\CIRC(L,c)$
  \item $\rw(\CIRC(L,c)) \leq 3\cdot 2^{3|L|}$
  \end{itemize}
\end{lemma}
\iflong\begin{proof}
  Since every \DS{} can be easily transformed into a \DL{} with the
  same number of terms, it suffices to show the lemma for \DL{}s. Let
  $L$ be a \DL{} with rules
  $(r_1=(t_1,c_1),\dotsc,r_\ell=(t_\ell,c_\ell))$.
  We construct the circuit $D=\CIRC(L,c)$ as follows.
  $D$ contains one input gate $g_f$
  and one NOT-gate $\overline{g_f}$, whose only incoming arc is from $g_f$,
  for every feature in $\feat(L)$. Furthermore, for every rule
  $r_i=(t_i,c_i)$, $D$ contains an AND-gate $g_{r_i}$, whose
  in-neighbours are the literals in $t_i$, i.e., if $t_i$ contains a
  literal $f=0$, then $g_{r_i}$ has $\overline{g_f}$ as an in-neighbour
  and if $t_i$ contains a literal $f=1$, then $g_{r_i}$ has $g_f$ as
  an in-neighbour. We now split the sequence
  $\rho=(r_1=(t_1,c_1),\dotsc,r_\ell=(t_\ell,c_\ell))$ into
  (inclusion-wise) maximal consecutive subsequences $\rho_i$ of rules
  that have the same class. Let $(\rho_1,\dotsc,\rho_r)$ be the
  sequence of subsequences obtained in this manner, i.e.,
  $\rho_1 \concat \rho_2 \concat \dotsb \concat \rho_r=\rho$, every
  subsequence $\rho_i$ is non-empty and contains only rules from one
  class, and the class of any rule in $\rho_i$ is different from the
  class of any rule in $\rho_{i+1}$ for every $i$ with $1 \leq i <
  r$. Now, for every subsequence $\rho_i$, we add an OR-gate
  $g_{\rho_i}$ to $D$, whose in-neighbours are
  the gates $g_{r}$ for every rule $r$ in $\rho_i$. Let $C$ be the set
  of all subsequences $\rho_i$ that only contains rules with class $c$
  and let $\overline{C}$ be the set of all other subsequences, i.e.,
  those that contain only rules whose class is not equal to $c$.
  For every subsequence $\rho$ in $\overline{C}$, we add a NOT-gate
  $\overline{g_\rho}$ to $D$, whose in-neighbour is the gate $g_\rho$.
  Moreover, for every subsequence $\rho_i$ in $C$, we add an AND-gate
  $g_{\rho_i}^A$ to $D$ whose in-neighbours are $g_{\rho_i}$ as well as
  $\overline{g_{\rho_j}}$ for every $\rho_j \in \overline{C}$ with
  $j<i$. Finally, we add an OR-gate $o$ to $D$ that also serves as the
  output gate of $D$ and whose in-neighbours are all the gates
  $g_{\rho}^A$ for every $\rho \in C$. This completes the construction
  of $D$ and it is easy to see that $L(e)=c$ if and only if $e$
  satisfies $D$ for every example $e$, which shows that $D$ satisfies (1).
  Towards showing (2), let $G$ be the set of all gates in $D$ apart
  from the gates $o$ and $g_f$ and $\overline{g_f}$ for every feature
  $f$. Then, $|G|\leq 3|L|$ and $D\setminus G$ is a forest, which
  together with \Cref{lem:ranktree}
  implies (2).
\end{proof}\fi

\begin{lemma}\label{lem:dsle-trans-circ}
  Let $\MM \in \{\DSE,\DLE\}$. There is a polynomial-time algorithm that given an $\MM$ $\mathcal{L}$ and a
  class $c$ produces a circuit $\CIRC(\mathcal{L},c)$ such that:
  \begin{enumerate}[(1)]
  \item for every example $e$, it holds that $\mathcal{L}(e)=c$ if and only if $e$ satisfies $\CIRC(\mathcal{L},c)$
  \item $\rw(\CIRC(\mathcal{L},c)) \leq 3 \cdot 2^{3\sum_{L \in \mathcal{L}}|L|}$
  \end{enumerate}
\end{lemma}
\iflong\begin{proof}
  We obtain the circuit $\CIRC(\mathcal{L},c)$ from the (not
  necessarily disjoint) union of the circuits $\CIRC(L,c)$, which are
  provided in \Cref{lem:dsl-trans-circ},
  for every $L \in \mathcal{L}$ after adding a new MAJ-gate
  with threshold $\lfloor|\mathcal{L}|/2\rfloor+1$,
  which also serves as the output gate of $\CIRC(\mathcal{L},c)$, that
  has one incoming arc from the output gate of $\CIRC(L,c)$
  for every $L \in \mathcal{L}$. Clearly,
  $\CIRC(\mathcal{L},c)$ satisfies (1). Moreover, to see that it also satisfies
  (2), recall that every circuit $\CIRC(L,c)$ has only
  $3|L|$ gates apart from the input gates, the NOT-gates
  connected to the input gates, and the output gate. Therefore, after
  removing $3|L|$ gates from every circuit
  $\CIRC(L,c)$ inside $\CIRC(\mathcal{L},c)$, the remaining
  circuit is a tree, which together with \Cref{lem:ranktree} implies (2).
\end{proof}\fi
The following corollary now follows immediately from
\Cref{lem:dsle-trans-circ} and \Cref{the:solve-circ}.
\fi
\begin{corollary}\label{cor:ds-ensa}
  Let $\MM \in \{\DSE,\DLE\}$ and let  $\PP \in \{\LAEX, \LCEX, \GAEX, \GCEX\}$. 
  \MM-\MPP{}$(\enssize+\termselem)$ is in \FPT{}.
\end{corollary}
Unlike, \DT{}s, where \DT{}-\MLCEX{} is solvable in polynomial-time,
this is not the case for \DS{}-\MLCEX{}. Nevertheless, we are able to
provide the following result, which shows that \DS{}-\MLCEX{} (and
even \DL{}-\MLCEX{}) is fixed-parameter tractable parameterized by
\termsize and \xpsize. The algorithm is based on a novel
characterisation of local contrastive explanations for \DL{}s.

\begin{lemma}\label{lem:dl-mlcex-branch}
    Let $\MM \in \{\DS,\DL\}$. \MM{}-\MLCEX{} for 
    $M \in \MM{}$ 
    can be solved in time $\bigoh(a^k\som{M}^2)$, where $a$ is equal
    to \termsize and $k$ is equal to \xpsize{}.
\end{lemma}
\iflong\begin{proof}\fi
\ifshort\begin{proof}[Proof Sketch]\fi  
  Since any \DS{} can be easily translated into a \DL{} without
  increasing the size of any term, it suffices to show the lemma for \DL{}s.  
  Let $(L,e,k)$ be an instance of \DL{}-\MLCEX{}, where
  $L=(r_1=(t_1,c_1),\dotsc,r_\ell=(t_\ell,c_\ell))$ is a \DL{},
   and let $r_i$ be the rule that
  classifies $e$, i.e., the first rule that applies to $e$.

  Let $R$ be the set of all rules $r_j$ of $L$ with $c_j\neq c_i$.
  For a subset $A \subseteq \feat(L)$, let
  $e_A$ be the example obtained from $e$ after setting
  $e_A(f)=1-e(f)$ for every $f \in A$. 
  We claim that:
  \begin{enumerate}[(1)]
  \item For every $r \in R$ and every set $A \subseteq \feat(L)$
    such that $e_A$ is classified by $r$, it holds that $A$ is a local
    contrastive explanation for $e$ w.r.t. $L$.
  \item Every local contrastive explanation $A$ for $e$
    w.r.t. $L$ contains a subset $A' \subseteq A$ for which 
    there is a rule $r \in R$ such that $e_{A'}$ is
    classified by $r$.
  \end{enumerate}
  \ifshort
    Because of (1) and (2), it holds that a set $A \subseteq
    \feat(L)$ is a local contrastive explanation if and only if
    there is a rule $r \in R$ such that $e_A$ is classified by
    $r$. Therefore, it is sufficient to be able to compute
    a minimum set of features $A$ such that $e_A$ is classified by $r$
    for every rule $r \in R$, which can be 
    achieved via a bounded-depth branching algorithm.
  \fi
  \iflong

  Towards showing (1), let $r \in R$ and let $A \subseteq \feat(L)$
  such that $e_A$ is classified by $r$. Then, $e_A$ differs from $e$
  only on the features in $A$ and moreover $M(e)\neq M(e_A)$ because
  $r \in R$. Therefore, $A$ is a local contrastive explanation for $e$
  w.r.t. $L$.

  Towards showing (2), let $A \subseteq \feat(A)$ be a local
  contrastive explanation for $e$ w.r.t. $L$. Then, there is an
  example $e'$ that differs from $e$ only on some set $A'\subseteq A$
  of features such
  that $e'$ is classified by a rule $r \in R$. Then, $e_{A'}$ is
  classified by $r$, showing (2).
    
  \begin{algorithm}[tb]
    \caption{Algorithm used in \Cref{lem:dl-mlcex-branch} to compute a
      local contrastive explanation for example $e$ w.r.t. a \DL{} $L$
      of size at most $k$ if such an explanation exists. The algorithm
      uses the function \textbf{findLCEXForRule}($L$, $e$, $r_j$) as a
      subroutine, which is illustrated in \Cref{alg:findLCEXForRule}}\label{alg:findLCEX}
    \small
    \begin{algorithmic}[1]
      \INPUT \DL{} $L=(r_1=(t_1,c_1),\dotsc,r_\ell=(t_\ell,c_\ell))$,
      example $e$ and integer $k$ (global variable).
      \OUTPUT return a cardinality-wise minimum local contrastive
      explanation $A \subseteq \feat(L)$ with $|A|\leq k$ for $e$
      w.r.t. $L$ or \NULL{} if such an explanation does not exists.
      \Function{\textbf{findLCEX}}{$L$, $e$}
      \State $r_i \gets$ rule of $L$ that classifies $e$
      \State $R \gets \SB r_j \in L \SM c_j\neq c_i\}$
      \State $A_b\gets \NULL$
      \For{$r_j \in R$}
      \State $A \gets$ \Call{findLCEXForRule}{$L$, $e$, $r_j$}
      \If{$A \neq \NULL$ and ($A_b=\NULL$ or $|A_b|>|A|$)}
      \State $A_b\gets A$
      \EndIf
      \EndFor
      \State \Return $A_b$
      \EndFunction
    \end{algorithmic}
  \end{algorithm}
  \begin{algorithm}[tb]
    \caption{Algorithm used as a subroutine in \Cref{alg:findLCEX}
      to compute a smallest set $A$
      of at most $k$ features such that $e_A$ is classified by the
      rule $r_j$ of a \DL{} $L$.}\label{alg:findLCEXForRule}
    \small
    \begin{algorithmic}[1]
      \INPUT \DL{} $L=(r_1=(t_1,c_1),\dotsc,r_\ell=(t_\ell,c_\ell))$,
      example $e$, rule $r_j$, and integer $k$ (global variable).
      \OUTPUT return a smallest set $A \subseteq \feat(L)$ with $|A|\leq k$
      such that $e_A$ is classified by $r_j$ if such a set exists and
      otherwise return \NULL{}.
      \Function{\textbf{findLCEXForRule}}{$L$, $e$, $r_j$}
      \State $A_0 \gets \SB f \SM (f=1-e(f)) \in t_j\SE$
      \State \Return \Call{findLCEXForRuleRec}{$L$, $e$, $r_j$, $A_0$}
      \EndFunction
      \Function{\textbf{findLCEXForRuleRec}}{$L$, $e$, $r_j$, $A'$}
      \If{$|A'|>k$}\label{algLLCEXRecOne}
      \State \Return \NULL
      \EndIf
      \State $r_\ell\gets$ any rule $r_\ell$ with $\ell<j$ that is satisfied by $e_{A'}$
      \If{$r_\ell=\NULL{}$}
      \State \Return $A'$\label{algLLCEXRecTwo}
      \EndIf
      \State $F_j\gets \SB f \SM (f=b) \in t_j \land b \in \{0,1\}\SE$
      \State $B \gets \SB f \SM (f=e_{A'}) \in t_\ell\SE\setminus (A'\cup F_j)$\label{algLLCEXRecThree}
      \If{$B=\emptyset$}
      \State \Return \NULL\label{algLLCEXRecFour}
      \EndIf
      \State $A_b\gets \NULL$
      \For{$f \in B$}
      \State $A\gets $\Call{findLCEXForRuleRec}{$L$, $e$, $r_j$, $A'\cup \{f\}$}\label{algLLCEXRecFive}
      \If{$A\neq \NULL$ and $|A|\leq k$}
      \If{($A_b==\NULL$ or $|A_b|>|A|$)}
      \State $A_b \gets A$
      \EndIf
      \EndIf
      \EndFor
      \Return $A_b$\label{algLLCEXRecSix}
      \EndFunction
    \end{algorithmic}
  \end{algorithm}

  Note that due to (1) and (2), we can compute a smallest local
  contrastive explanation for $e$ w.r.t. $L$ by the algorithm
  illustrated in \Cref{alg:findLCEX}. That is, the algorithm computes
  the smallest set $A$ of at most $k$ features such that $e_A$ is
  classified by $r_j$ for every rule $r_j \in R$. It then, returns the
  smallest such set over all rules $r_j$ in $R$ if such a set existed
  for at least one of the rules in $R$. The main ingredient of the
  algorithm is the function \Call{findLCEXForRule}{$L$, $e$, $r_j$},
  which is illustrated in \Cref{alg:findLCEXForRule}, and
  that for a rule $r_j \in R$ computes the smallest set $A$ of at most
  $k$ features such that $e_A$ is classified by $r_j$. The function
  \Cref{alg:findLCEXForRule} achieves this as follows.
  It first computes the set $A_0=\SB f \SM (f=1-e(f)) \in t_j\SE$ of
  features that need to be part of $A$ in order for $e_A$ to satisfy
  $r_j$, i.e., all features $f$ where $e(f)$ differs from the literal
  in $t_j$. It then computes the set $A$ recursively via the
  (bounded-depth) branching algorithm given in
  \Call{findLCEXForRuleRec}{$L$, $e$, $r_j$, $A$}, which given a set
  $A'$ of features such that $e_{A'}$ satisfies $r_j$ computes a
  smallest extension (superset) $A$ of $A'$ of size at most $k$ such
  that $e_A$ is classified by $r_j$. To do so the function first
  checks in \Cref{algLLCEXRecOne} whether $|A'|>k$ and if so correctly returns
  \NULL{}. Otherwise, the function checks whether there is any rule $r_\ell$ with
  $\ell<j$ that is satisfied by $e_{A'}$. If that is not the case, it
  correctly returns $A'$ (in \Cref{algLLCEXRecTwo}). Otherwise, the
  function computes the set $B$ of features in \Cref{algLLCEXRecThree} that occur in $t_\ell$ --
  and therefore can be used to falsify $r_\ell$ -- but do not occur in
  $A'$ or in $t_j$. Note that we do not want to include any feature in
  $A'$ or $t_j$ in $B$ since changing those features would either
  contradict previous branching decisions made by our algorithm or it
  would prevent us from satisfying $r_j$. The function then returns
  \NULL{} if the set $B$ is empty in \Cref{algLLCEXRecFour}, since in this case it is no longer
  possible to falsify the rule $r_\ell$. Otherwise, the algorithm
  branches on every feature in $f \in B$ and tries to extend $A'$ with
  $f$ using a recursive call to itself with $A'$ replaced by $A'\cup
  \{f\}$ in \Cref{algLLCEXRecFive}. It then returns the best solution
  found by any of those recursive calls in \Cref{algLLCEXRecSix}.
  This completes the description of the algorithm, which can be easily
  seen to be correct using (1) and (2).

  We are now ready to analyse the runtime of the algorithm, i.e.,
  \Cref{alg:findLCEX}. First note that all operations in
  \Cref{alg:findLCEX} and \Cref{alg:findLCEXForRule} that are not recursive
  calls take time at most $\bigoh(\som{L})$. We start by 
  analysing the run-time of the function \Call{findLCEXForRule}{$L$,
    $e$, $r_j$}, which is at most $\bigoh(\som{L})$ times the number
  of recursive calls to the function
  \Call{findLCEXForRuleRec}{$L$, $e$, $r_j$, $A$}, which in turn can be easily
  seen to be at most $a^k$ since the function branches into at most
  $a$ branches at every call and the depth of the branching is at most $k$.
  Therefore, we obtain $\bigoh(a^k\som{L})$ as the total runtime of
  the function \Call{findLCEXForRule}{$L$, $e$, $r_j$}. Since this
  function is called at most $\som{L}$ times from the function
  \Call{findLCEX}{$L$, $e$}, this implies
  a total runtime of the algorithm of $\bigoh(a^k\som{L}^2)$.\fi
\end{proof}



The following lemma is now a natural extension of
\Cref{lem:dl-mlcex-branch} for ensembles of \DL{}s.
\begin{lemma}\label{lem:dle-mlcex-branch}
    Let $\MM \in \{\DSE,\DLE\}$. \MM{}-\MLCEX{} for $M \in \MM$
    can be solved in time $\bigoh(m^sa^k\som{M}^2)$, where $m$ is
    \termselem, $s$ is \enssize, $a$ is \termsize, and $k$ is \xpsize.
\end{lemma}
\iflong\begin{proof}
  Since any \DSE{} can be easily translated into a \DLE{} without
  increasing the size of any term and where every ensemble element has
  at most one extra rule using the empty term, it suffices to show the lemma for~\DLE{}s.  
  The main ideas behind the algorithm, which is illustrated in \Cref{alg:findELCEX}, are similar to the ideas behind
  the algorithm used in \Cref{lem:dl-mlcex-branch} for a single \DL{}.

  \begin{algorithm}[tb]
    \caption{Algorithm used in \Cref{lem:dle-mlcex-branch} to compute a
      local contrastive explanation for example $e$ w.r.t. a \DLE{}
      $\dle=\{\dl_1,\dotsc,\dl_\ell\}$
      with $\dl_i=(r_1^i=(t_1^i,c_1^i),\dotsc,r_\ell^i=(t_\ell^i,c_\ell^i))$
      of size at most $k$ if such an explanation exists. The algorithm
      uses the function \textbf{findELCEXForRules}($\dle$, $e$, $(r_{j_1}^1,\dotsc,r_{j_\ell}^\ell)$) as a
      subroutine, which is illustrated in \Cref{alg:findELCEXForRules}.}\label{alg:findELCEX}
    \small
    \begin{algorithmic}[1]
      \INPUT \DLE{} $\dle=\{\dl_1,\dotsc,\dl_\ell\}$ with $\dl_i=(r_1^i=(t_1^i,c_1^i),\dotsc,r_\ell^i=(t_\ell^i,c_\ell^i))$,
      example $e$ and integer $k$ (global variable).
      \OUTPUT return a cardinality-wise minimum local contrastive
      explanation $A \subseteq \feat(\dle)$ with $|A|\leq k$ for $e$
      w.r.t. $\dle$ or \NULL{} if such an explanation does not exists.
      \Function{\textbf{findELCEX}}{$\dle$, $e$}
      \For{$o \in [k]$}
      \State $R^o \gets \SB r \in \dl_o \SE$
      \EndFor

      \State $A_b\gets \NULL$
      \For{$(r_{j_1}^1,\dotsc,r_{j_\ell}^\ell) \in R^1\times \dotsb
        \times R^\ell$}\label{algLELCEXOne}
      \State $n_{\neq}\gets |\SB o \in [\ell] \SM c_{j_o}^o\neq
      \dle(e)\SE|$
      \State $n_= \gets |\SB o \in [\ell] \SM c_{j_o}^o= \dle(e)\SE|$
      \If{$n_{\neq} > n_=$}\label{algLELCEXTwo}
      \State $A \gets$ \Call{findELCEXR}{$\dle$, $e$, $(r_{j_1}^1,\dotsc,r_{j_\ell}^\ell)$}
      \If{$A \neq \NULL$ and ($A_b=\NULL$ or $|A_b|>|A|$)}
      \State $A_b\gets A$
      \EndIf
      \EndIf
      \EndFor
      \State \Return $A_b$
      \EndFunction
    \end{algorithmic}
  \end{algorithm}

  \begin{algorithm}[tb]
    \caption{Algorithm used as a subroutine in \Cref{alg:findELCEX}
      to compute a smallest set $A$
      of at most $k$ features such that $e_A$ is classified by the
      rule $r_j^o$ for every \DL{} $\dl_o$.}\label{alg:findELCEXForRules}
    \small
    \begin{algorithmic}[1]
      \INPUT \DL{} \DLE{} $\dle=\{\dl_1,\dotsc,\dl_\ell\}$ with $\dl_i=(r_1^i=(t_1^i,c_1^i),\dotsc,r_\ell^i=(t_\ell^i,c_\ell^i))$,
      example $e$, rules $(r_{j_1}^1,\dotsc,r_{j_\ell}^\ell)$, and integer $k$ (global variable).
      \OUTPUT return a smallest set $A \subseteq \feat(L)$ with $|A|\leq k$
      such that $e_A$ is classified by $r_{j_o}^o$ for every \DL{} $\dl_o$ if such a set exists and
      otherwise return \NULL{}.
      \Function{\textbf{findELCEXR}}{$\dle$, $e$, $(r_{j_1}^1,\dotsc,r_{j_\ell}^\ell)$}
      \State $A_0 \gets \SB f \SM (f=1-e(f)) \in t_{j_o}^o \land o \in
      [\ell]\SE$
      \State \Return \Call{findELCEXRR}{$\dle$, $e$, $(r_{j_1}^1,\dotsc,r_{j_\ell}^\ell)$, $A_0$}
      \EndFunction
      \Function{\textbf{findELCEXRR}}{$\dle$, $e$, $(r_{j_1}^1,\dotsc,r_{j_\ell}^\ell)$, $A'$}
      \If{$|A'|>k$}\label{algLELCEXRecOne}
      \State \Return \NULL
      \EndIf
      \State $r_\ell^o\gets$ any rule with $\ell<j_o$ that
      is satisfied by $e_{A'}$ for $o \in [\ell]$
      \If{$r_\ell^o=\NULL{}$}
      \State \Return $A'$\label{algLELCEXRecTwo}
      \EndIf
      \State $F'\gets \SB f \SM (f=b) \in t_{j_o}^o \land o \in [\ell] \land b \in \{0,1\}\SE$
      \State $B \gets \SB f \SM (f=e_{A'}) \in t_\ell^o\SE\setminus (A'\cup F')$\label{algLELCEXRecThree}
      \If{$B=\emptyset$}
      \State \Return \NULL\label{algLELCEXRecFour}
      \EndIf
      \State $A_b\gets \NULL$
      \For{$f \in B$}
      \State $A\gets $\Call{findELCEXRR}{$\dle$, $e$, $(r_{j_1}^1,\dotsc,r_{j_\ell}^\ell)$, $A'\cup \{f\}$}\label{algLELCEXRecFive}
      \If{$A\neq \NULL$ and $|A|\leq k$}
      \If{($A_b==\NULL$ or $|A_b|>|A|$)}
      \State $A_b \gets A$
      \EndIf
      \EndIf
      \EndFor
      \Return $A_b$\label{algLELCEXRecSix}
      \EndFunction
    \end{algorithmic}
  \end{algorithm}

  Let $(\dle,e,k)$ be an instance of \DLE{}-\MLCEX{} \DLE{} with
  $\dle=\{\dl_1,\dotsc,\dl_\ell\}$ and
  $\dl_i=(r_1^i=(t_1^i,c_1^i),\dotsc,r_\ell^i=(t_\ell^i,c_\ell^i))$
  for every $i \in [\ell]$. First note that if $A$ is a local contrastive explanation
  for $e$ w.r.t. $\dle$, then there is an example $e'$ that differs
  from $e$ only on the features in $A$ such that $\dle(e)\neq
  \dle(e')$. Clearly $e'$ is classified by exactly one rule of every
  \DL{} $\dl_i$. The first idea behind \Cref{alg:findELCEX} is
  therefore to enumerate all possibilities for the rules that classify
  $e'$; this is done in \Cref{algLELCEXOne} of the algorithm. Clearly,
  only those combinations of rules are relevant that lead to a
  different classification of $e'$ compared to $e$ and this is ensured
  in \Cref{algLELCEXTwo} of the algorithm. For every such combination $(r_{j_1}^1,\dotsc,r_{j_\ell}^\ell)$
  of rules the algorithm then calls the subroutine
  \Call{findELCEXR}{$\dle$, $e$,
    $(r_{j_1}^1,\dotsc,r_{j_\ell}^\ell)$}, which is illustrated in
  \Cref{alg:findELCEXForRules}, and computes the smallest set $A$ of
  at most $k$ features such that $e_A$ is classified by the rule
  $r_{j_o}^o$ of $\dl_o$ for every $o \in [\ell]$. This subroutine
  and the proof of its correctness work very similar to the subroutine \textbf{findLCEXForRule}($L$,
  $e$, $r_j$) of \Cref{alg:findLCEXForRule} used in
  \Cref{lem:dl-mlcex-branch} for a single \DL{} and is therefore not
  repeated here. In essence the
  subroutine branches over all possibilities for such a set $A$.

  We are now ready to analyse the runtime of the algorithm, i.e.,
  \Cref{alg:findELCEX}. First note that the function
  \Call{findELCEX}{$\dle$, $e$} makes at most $m^s$ calls to the
  subroutine \Call{findELCEXR}{$\dle$, $e$,
    $(r_{j_1}^1,\dotsc,r_{j_\ell}^\ell)$} and apart from those calls
  all operations take time at most $\bigoh(\som{\dle})$. Moreover, the
  runtime of the subroutine \Call{findELCEXR}{$\dle$, $e$,
    $(r_{j_1}^1,\dotsc,r_{j_\ell}^\ell)$} is the same as the runtime
  of the subroutine \Call{findLCEXForRule}{$L$, $e$, $r_j$} given in
  \Cref{alg:findLCEXForRule}, i.e.,
  $\bigoh(a^k\som{\dle}^2)$. Therefore, we obtain
  $\bigoh(m^sa^k\som{\dle}^2)$ as the total run-time of the algorithm.
\end{proof}\fi

The following result now follows immediately from \Cref{lem:dle-mlcex-branch}.

\begin{corollary}\label{cor:DS-MLCEX-FPT-const-ens}
  Let $\MM \in \{\DS,\DL\}$. \MM-\MLCEX{}$(\termselem+\xpsize)$ is
  in \FPT{}, when \enssize{} is constant.
\end{corollary}

\section{Hardness Results}\label{sec:hardness}

In this section, we provide our algorithmic lower
bounds. We start by showing a close connection between the complexity
of all of our explanation problems to the following two problems. As
we will see the hardness of finding explanations comes from the
hardness of deciding whether or not a given model classifies all
examples in the same manner. More specifically, from the \HOM{}
problem defined below, which asks whether a given model has an example that is classified
differently from the all-zero example, i.e., the example being $0$ on
every feature. We also need the \kHOM{} problem, which is a parameterized
version of \HOM{} that we use to show parameterized hardness results
for deciding the existence of local contrastive explanations.

In the following, let $\MM$ be a model type.
\pbDef{$\MM{}$-\textsc{Homogeneous (\HOM{}})}{A model $M \in \MM{}$.}{
Is there an example $e$ such that $M(e) \neq M(e_0)$, where $e_0$ is the
all-zero example?}

\pbDef{$\MM$-p-\textsc{Homogeneous (\kHOM{}})}{A model $M \in \MM{}$ and
  integer $k$.}{
Is there an example $e$ that sets at most $k$ features to $1$ such that $M(e) \neq M(e_0)$, where $e_0$ is the all-zero example?}

The following lemma now shows the connection between \HOM{} and the
considered explanation problems.

\begin{lemma}\label{lem:HOM-reduction}
  Let $M \in \MM$ be a model, $e_0$ be the all-zero example, and let $c=M(e_0)$.
  The following statements are equivalent:  
  \begin{enumerate}[(1)]
  \item $M$ is a no-instance of \MM{}-\HOM{}.
  \item The empty set is a solution for the instance $(M,e_0)$ of
    \MM{}-\SMLAEX{}.
  \item $(M,e_0)$ is a no-instance of \MM{}-\SMLCEX{}.    
  \item The empty set is a solution for the instance $(M,c)$ of \MM{}-\SMGAEX{}.
  \item The empty set is a solution for the instance $(M,1-c)$ of
    \MM{}-\SMGCEX{}.
  \item $(M,e_0,0)$ is a yes-instance of \MM{}-\MLAEX{}.
  \item $(M,e_0)$ is a no-instance of \MM{}-\MLCEX{}.    
  \item $(M,c,0)$ is a yes-instance of \MM{}-\MGAEX{}.
  \item $(M,1-c,0)$ is a yes-instance of \MM{}-\MGCEX{}.
  \end{enumerate}
\end{lemma}
\begin{proof}
  It is easy to verify that all of the statements (1)--(9) are
  equivalent to the following statement (and therefore equivalent to
  each other): $M(e)=M(e_0)=c$ for every example $e$. 
\end{proof}
While \Cref{lem:HOM-reduction} is sufficient for most of our hardness
results, we also need the following lemma to show certain
parameterized hardness results for deciding the existence of local
contrastive explanations.
\begin{lemma}\label{lem:kHOM-reduction}
    Let $M \in \MM$ be a model and let $e_0$ be the all-zero example.
    The following problems are equivalent:  
    \begin{enumerate}[(1)]
    \item $(M,k)$ is a yes-instance of \MM{}-\kHOM{}.
    \item $(M,e_0,k)$ is a yes-instance of \MM{}-\MLCEX{}.
    \item $(M,e_0,k)$ is a yes-instance of \MM{}-\MLCEX{}.
  \end{enumerate}
\end{lemma}
\iflong\begin{proof}
    The lemma follows because both statements are equivalent to
    the following statement:
    There is an example $e$ that sets at most $k$ features to $1$ such
    that $M(e)\neq M(e_0)$. 
\end{proof}\fi

We will often reduce from the following problem, which is well-known
to be \NPh{} and also \Wh{1} parameterized by $k$~\cite{DowneyFellows13}.

\pbDef{{\sc Multicoloured Clique} (\MCC)}
{A graph $G$ with a proper $k$-colouring of $V(G)$.}
{Is there a clique of size $k$ in $G$?}

  \newcommand{\dup}{set-modelling}
  \newcommand{\dupa}{(a)-\dup}
  \newcommand{\dupab}{(a,b)-\dup}
  \newcommand{\dupc}{(2)-\dup}
  \newcommand{\dupcc}{(2,1)-\dup}
  \newcommand{\pies}{subset-modelling}
  \newcommand{\piesa}{(a)-\pies}
  \newcommand{\piesab}{(a,b)-\pies}
  \newcommand{\piesc}{(2)-\pies}
  \newcommand{\piescc}{(2,1)-\pies}


\subsection{A Meta-Theorems for Models Explanations}\label{ss:meta_hardness}
    We now introduce the notions of \pies{} and \dup{}, which are structural properties of classes of models.  
    Intuitively, these properties capture whether certain canonical models can be efficiently constructed within a given class.
    In \Cref{lem:meq_dt,lem:mss_ds_dl},  
    we show that all the considered model classes satisfy at least one of these properties.  
    Moreover, together with the
    hardness results in
    \Cref{lem:new_wh1,lem:new_wh2,lem:new_ens_wh1}, these properties 
    account for most, but not all, of the hardness results presented in this paper.

  Let $\FFF$ be a family of sets of features and let $c \in \{0,1\}$.
  Let $M_{\FFF,c}^=$ be a model such that
  for every example $e$
  with $\bigcup \FFF \subseteq F(e)$,
  $M_{\FFF,c}^=$ $c$-classifies $e$ if and only if 
  there exists $F \in \FFF$
  such that $e(f) = 1 \Leftrightarrow f \in F$ for every feature $f \in (\bigcup \FFF)$.
  Let $M_{\FFF,c}^\subseteq$ be a model such that
  for every example $e$ with $\bigcup \FFF \subseteq F(e)$,
  $M_{\FFF,c}^\subseteq$ $c$-classifies $e$ if and only if 
  there exists $F \in \FFF$ 
  such that $e(f) = 1$ for every $f \in F$.

  We say that class of models $\MM$ is \emph{\dupab}
  if and only if 
  for every 
   $c \in \{0,1\}$ and
   every family $\FFF$ of set of features
  such that $\max_{F \in \FFF} |F| \leq a$ and $|\FFF|\leq b$,
  we can construct $M \in \MM$
  in time $\poly(a + b)$,
  such that for every example $e$ with $\bigcup \FFF \subseteq F(e)$,
  $M(e) = M_{\FFF,c}^=(e)$.
  We say that a class of models $\MM$ is \emph{\dupa}
  if and only if for every $b$,  $\MM$ is \emph{\dupab}.
  We say that a class of models $\MM$ is \emph{\dup}
  if and only if for every $a$,  $\MM$ is \emph{\dupa}.
  We analogously define \emph{\piesab}, \emph{\piesa} and
  \emph{\pies} using $M_{\FFF,c}^\subseteq$ instead of $M_{\FFF,c}^=$.

\begin{lemma}\label{lem:new_wh2}
  Let $\MM^=$ and $\MM^\subseteq$ be classes of models such that $\MM^=$ is \dup{} and $\MM^\subseteq$ is \pies{}.
  Let $\PP \in \{\MLAEX, \MGAEX, \MGCEX, \MLCEX\}$. 
  Then the following hold:
  \begin{itemize}
      \item $\MM^{=}$-\MLAEX{} is \NPh{};
      \item $\MM^{=}$-\MLAEX{}$(\xpsize{})$ is \Wh{2};
      \item $\MM^{\subseteq}$-$\PP$ is \NPh{};
      \item $\MM^{\subseteq}$-$\PP(\xpsize{})$ is \Wh{2}.
  \end{itemize}
\end{lemma}
\begin{proof}    

  We provide a parameterized reduction from the \textsc{Hitting Set}
  problem, which is well-known to be \NPh{} and \Wh{2}
  parameterized by the size of the
  solution~\cite{DowneyFellows13}. That is, given a family $\FFF$
  of sets over a universe $U$ and an integer $k$, the
  \textsc{Hitting Set} problem is to decide whether $\FFF$ has a
  \emph{hitting set} of size at most $k$, i.e., a subset $S \subseteq
  U$ with $|S|\leq k$ such that $S\cap F\neq \emptyset$ for every $F
  \in \FFF$.
  Let $B$ be the set of features, containing one binary feature 
  $f_u$ for each $u \in \bigcup \FFF$. 
  Define 
  $\BBB = \SB \SB f_u\SM  u \in F \SE \SM  F \in \FFF \SE$
  as a family of subsets of $B$. 
  Let $M^=_{\BBB,1} \in \MM^=$ 
  ($M^\subseteq_{\BBB,1} \in \MM^\subseteq$)
  be defined over set of features $B$.
  Note that we can construct 
  $M^=_{\BBB,1}$ ($M^\subseteq_{\BBB,1}$)
  in $\poly(||\FFF||)$ time, because 
  $\MM^=$ is \dup{} ($\MM^\subseteq$ is \pies{}).

  We claim that $(M^=_{\BBB,1},e_\emptyset,k)$, $(M^\subseteq_{\BBB,1},e_\emptyset,k)$,
  $(M^\subseteq_{\BBB,1},e_B,k)$, $(M^\subseteq_{\BBB,1},0,k)$,
  $(M^\subseteq_{\BBB,1},1,k)$ are instances of \MLAEX{},
  \MLAEX{}, \MLCEX{}, \MGAEX{}, and \MGCEX{}, respectively, that are
  equivalent to a
  given instance $(U,\FFF,k)$ of \textsc{Hitting Set}.

  Let $e_{B'}$ be the example 
  over $B$, where only features in $B'\subseteq B$ are positive,
  i.e., for each $f \in B$, 
  $e_{B'}(f) = 1$ if and only if $f \in B'$.
  In particular, $e_\emptyset$ is the all-zero example and 
  $e_B$ is the all-one example.
  For every pair of examples $e,e' \in E(B)$,
  we say that $e \leq e'$ if and only if
  for each $f \in B$, $e(f) \leq e'(f)$.
  
  Let $E_\BBB = \SB e_{B'} \SM B' \in \BBB \SE$.
  Observe that:
  \begin{enumerate}
  \item[(1)] For every example $e \in E(B)$, the following holds:
    \begin{itemize}
    \item $M^=_{\BBB,1}(e)=1$ if and only if $e \in E_\BBB$ and
    \item $M^\subseteq_{\BBB,1}(e)=1$ if and only if
      there exists $e' \in E_\BBB$ with $e \geq e'$.
    \end{itemize}
  \item[(2)] For every two examples $e,e' \in E(B)$, it holds that
    if $e\leq e'$, then $M^\subseteq_{\BBB,1}(e) \leq M^\subseteq_{\BBB,1}(e')$.
  \end{enumerate}

  We start by showing that $(U,\FFF,k)$ is a yes-instance of
  \textsc{Hitting Set} if and only if $(M^=_{\BBB,1},e_\emptyset,k)$/$(M^\subseteq_{\BBB,1},e_\emptyset,k)$ is a yes-instances of \MLAEX{}.
  Towards showing the forward direction, let $S_U$ be a hitting set for
  $\FFF$ of size at most $k$ and let $S = \SB f_u \SM u \in S_U \SE$.
  Note that for every $e \in E(B)$,
  $e \leq e_{B\setminus S}$
  if and only if 
  $e$ agrees with $e_\emptyset$ on all features in $S$.
  Since $S_U$ is a hitting set,
  we obtain that
  for every $e \in \SB e_{B'} \SM B' \in \BBB \SE$,
  we have
  $e \not\leq e_{B\setminus S}$.
  Therefore, $S$ is a local abductive explanation for $e_\emptyset$ w.r.t. $M^=_{\BBB,1}$.
  Additionally, using (2), we conclude that 
  $S$ is a local abductive explanation for $e_\emptyset$ w.r.t. $M^\subseteq_{\BBB,1}$.

  Towards showing the reverse direction, let $S$ be a local abductive  
  explanation of size at most $k$ for $e_\emptyset$ w.r.t. $M^=_{\BBB,1}$.
  We claim that $S_U=\SB u \SM f_u \in S\SE$ is a hitting set for $\FFF$.
  Suppose otherwise, and let $F' \in \FFF$ be a set such that 
  $F'\cap S_U=\emptyset$.
  Let $B' = \SB f_u \SM u \in F'\SE$ then, 
  the example $e_{B'}$
  agrees with $e_\emptyset$ on every
  feature in $S$.
  However, from (1) we have
  $M^=_{\BBB,1}(e_{B'})=1\neq M^=_{\BBB,1}(e_{\emptyset})$,
  contradicting our assumption that $S$ is
  a local abductive explanation for $e_\emptyset$ w.r.t. $M^\subseteq_{\BBB,1}$.
  By an analogous argument, we conclude that $S$ is also a 
  local abductive explanation for $e_\emptyset$ w.r.t.
  $M^\subseteq_{\BBB,1}$.

  Next, we show the equivalence between the instances
  $(M^\subseteq_{\BBB,1},e_\emptyset,k)$ and
  $(M^\subseteq_{\BBB,1},e_B,k)$ of \MLAEX{} and \MLCEX{}, respectively.
  That is, we claim that  
  $S$ is a local abductive explanation for $e_\emptyset$ w.r.t. $M^\subseteq_{\BBB,1}$  
  if and only if  
  $S$ is a local contrastive explanation for $e_B$ w.r.t. $M^\subseteq_{\BBB,1}$.  
  To prove the forward direction,  
  observe that  
  $e_{B\setminus S}$ agrees with $e_\emptyset$ on all features in $S$.  
  Since $M^\subseteq_{\BBB,1}(e_{B\setminus S}) = 0$  
  and $M^\subseteq_{\BBB,1}(e_B) = 1$,  
  it follows that $S$ satisfies the conditions for a local contrastive explanation.  
  For the reverse direction,  
  given that $M^\subseteq_{\BBB,1}(e_{B\setminus S}) = 0$ and using (2),  
  we conclude that for every example $e \in E(B)$,  
  if $e$ agrees with $e_\emptyset$ on all features in $S$,  
  then $M^\subseteq_{\BBB,1}(e) = 0$.  
  Thus, $S$ also satisfies the conditions for a local abductive explanation.

  We now show the case of \MGAEX{}. Namely, we claim that $S$ is a
  local abductive explanation for $e_\emptyset$
  w.r.t. $M^\subseteq_{\BBB,1}$ if and only if there is an assignment $\tau:
  S\rightarrow\{0,1\}$ that is a global abductive
  explanation for $0$ w.r.t. $M^\subseteq_{\BBB,1}$.

  Towards showing the forward direction, note that setting $\tau(s)=0$
  for every $s \in S$ shows the result.
  To prove the reverse direction,
  let $\tau:S\rightarrow \{0,1\}$ be 
  a global abductive explanation. First note that, because of (2), we can
  assume that $\tau(s)=0$ for every $s \in S$. 
  Therefore, $S$ is  local abductive explanation.
  
  The same reasoning applies for the case of 
  global contrastive explanations for $1$ w.r.t. $M^\subseteq_{\BBB,1}$.
\end{proof}

While \Cref{lem:new_wh2} addresses the hardness of explanation problems for single models and specific types of explanations,
\Cref{lem:new_ens_wh1,lem:new_wh1} deal with ensembles of models and the new problems \HOM{} and \kHOM{}.  
These results, combined with \Cref{lem:HOM-reduction,lem:kHOM-reduction},  
will later allow us to derive hardness results for explanation problems in the ensemble setting.
Moreover, \Cref{lem:new_ens_wh1,lem:new_wh1} establish hardness results under more restricted variants of \dup{} and \pies{},  
which will be used to ensure that certain model parameters (such as \sizeelem{}) remain bounded by constants.

\begin{lemma}\label{lem:new_ens_wh1}
    Let $\MM^=$ and $\MM^\subseteq$ be classes of models  
    where $\MM^=$ is \dupc{} and $\MM^\subseteq$ is \piesc{}.
    Let $g \in \{=,\subseteq\}$.
    Then the following hold:
  \begin{itemize}
      \item $\MM^g_\MAJ$-\HOM{} is \NPh{};
      \item $\MM^g_\MAJ$-\HOM{}$(\enssize{})$ is \Wh{1};
      \item $\MM^g_\MAJ$-\kHOM{}$(\enssize{} + k)$ is \Wh{1};
  \end{itemize}
\end{lemma}
\begin{proof}
  We give a parameterized reduction from \MCC{} that is also a
  polynomial-time reduction. That is, given an
  instance $(G,k)$ of \MCC{} with $k$-partition $(V_1,\dotsc,V_k)$ of $V(G)$, we
  will construct an ensemble $\EEE^=$ of type $\MM^{=}_\MAJ$  
  (an ensemble $\EEE^\subseteq$ of type $\MM^{\subseteq}_\MAJ$)
  with $|\EEE^=|= 2\binom{k}{2}-1$
  ($|\EEE^\subseteq|= 2k -1$)
  such
  that $G$ has a $k$-clique if and only if $\EEE^=$ 
  ($\EEE^\subseteq$) classifies at least
  one example positively.
  This will already suffice to show the stated
  results for $\MM{}^=_\MAJ$-\HOM{} ($\MM{}^\subseteq_\MAJ$-\HOM{}).
  Moreover, to show the results for
  $\MM{}^=_\MAJ$-\kHOM{} ($\MM{}^\subseteq_\MAJ$-\kHOM{}) 
  we additional show that $G$ has a $k$-clique if and only if
  $\EEE^=$ 
  ($\EEE^\subseteq$)
  classifies an example positively that sets at most $k$
  features to $1$.

  Both $\EEE^=$ and $\EEE^\subseteq$
  will use the set of features $F =\bigcup_{i \in [k]} F_i$, where $F_i = \SB f_{v} \SM v \in V_i\SE$.
  Moreover both $\EEE^=$ and $\EEE^\subseteq$
  will classify an example $e \in E(F)$ positively
  if and only if all the following:
  \begin{itemize}
      \item[(1)] For every $i \in [k]$, there exists at least one $f \in F_i$, such that $e(f) = 1$,
      \item[(2)] For every $i \in [k]$, there exists at most one $f \in F_i$, such that $e(f) = 1$,
      \item[(3)] For every $i,j \in [k]$, $v \in V_i$ and $u \in V_j$,
        if $i<j$ and $e(f_v) = e(f_u) = 1$, then $vu \in E(G)$.
  \end{itemize}

  For every $1\leq i<j \leq k$,
  let $\FFF_{i,j} = \SB \{f_v,f_u\} \SM uv \in E(G) \land v \in V_i \land v \in V_j \SE$.
  Note that, 
  for every example $e \in E(F)$,
  $e$ agrees with (1), (2) and (3) 
  if and only if for every $1\leq i<j \leq k$, 
  $M^=_{\FFF_{i,j},1}(e) = 1$.
  To construct $\EEE^=$ we will use $M^=_{\FFF_{i,j},1}$ for every $1\leq i<j \leq k$,
  and $\binom{k}{2}-1$ models $M^{=}_{\emptyset,0}$ that classifies all examples negatively.
  Therefore, $\EEE^=$ classifies an example $e \in E(F)$ positively
  if and only if
  all $M^=_{\FFF_{i,j},1}$ classifies $e$ positively.
  Therefore, $\EEE^=$ agrees with (1),(2) and (3).
  Moreover, $M^=_{\FFF_{i,j},1}$ and $M^{=}_{\emptyset,0}$
  are in $\MM^=$ and can therefore be constructed in $\poly(|V(G)|)$ time.

  For every $i \in [k]$,
  let $\FFF_i = \SB \{f\} \SM f \in F_i \SE$ 
  and
  let $\FFF = \SB \{f_v,f_u\} \SM uv \notin E(G) \SE$.
  Note that, 
  for every example $e \in E(F)$,
  $e$ agrees with (1) 
  if and only if for every $i \in [k]$, 
  $M^\subseteq_{\FFF_{i},1}(e) = 1$.
  Moreover,
  for every example $e \in E(F)$,
  $e$ agrees with (2) and (3) 
  if and only if 
  $M^{\subseteq}_{\FFF,0}(e) = 1$.
  To construct $\EEE^\subseteq$ we will use $M^\subseteq_{\FFF_{i},1}$ for every $i \in [k]$,
  $M^{\subseteq}_{\FFF,0}$ 
  and $k$ models $M^{\subseteq}_{\emptyset,0}$
  that classifies all examples negatively.
  Therefore, $\EEE^\subseteq$ classifies an example $e \in E(F)$ positively
  if and only if
  $M^{\subseteq,\bot}_\FFF$ and
  all $M^\subseteq_{\FFF_{i},1}$ classify $e$ positively.
  Therefore, $\EEE^\subseteq$ agrees with (1),(2) and (3).
  Moreover, $M^\subseteq_{\FFF_{i},1}$, $M^{\subseteq}_{\FFF,0}$ 
  and $M^{\subseteq}_{\emptyset,0}$
  are in $\MM^\subseteq$ and can therefore be constructed in $\poly(|V(G)|)$ time.

  Towards showing the
  forward direction, let $C=\SB v_1,\dotsc,v_k\SE$ be a $k$-clique of
  $G$, where $v_i \in V_i$ for every $i \in [k]$. We
  claim that the example $e$ that is $1$ exactly at the features
  $f_{v_1},\dotsc,f_{v_k}$ (and otherwise $0$) is classified
  positively by 
 $\EEE^=$ ($\EEE^\subseteq$).
  Note that $e$ agrees with (1), (2) and (3), 
  therefore $\EEE^=$ ($\EEE^\subseteq$) classifies $e$ positively.

  Towards showing the reverse direction, 
  suppose that there is an
  example $e$ that is classified positively by $\EEE^=$ ($\EEE^\subseteq$).
  Since $e$ agrees with (1) and (2),
  we obtain that for every $i\in [k]$,
  there exists exactly one $v_i \in V_i$ such that $e(f_{v_i}) = 1$.
  Moreover since $e$ agrees with (3)
  we obtain that for every $i,j\in [k]$,
  $v_i$ and $v_j$ are adjacent in $G$.
   Therefore,
  $C=\{v_1,\dotsc,v_k\}$ is a $k$-clique of $G$.
\end{proof}

\begin{lemma}\label{lem:new_wh1}
    Let $\MM^=$ and $\MM^\subseteq$ be classes of models  
    where $\MM^=$ is \dupcc{} and $\MM^\subseteq$ is \piescc{}.    
    Let $g \in \{=,\subseteq\}$. 
    Then the following hold:
    \begin{itemize}
        \item $\MM^g_\MAJ$-\HOM{} is \NPh{};
        \item $\MM^g_\MAJ$-\kHOM{}$(k)$ is \Wh{1}.
    \end{itemize}
\end{lemma}
\begin{proof}
  We give a parameterized reduction from \MCC{} that is also a
  polynomial-time reduction. That is, given an
  instance $(G,k)$ of \MCC{} with $k$-partition $\{V_1,\dotsc,V_k\}$ of $V(G)$, we
  will construct an ensemble $\MMM$ of type $\MM^{g}_\MAJ$  
  with $|\MMM|= 2n(\binom{n}{2}-m) + 2k -1$
  such
  that $G$ has a $k$-clique if and only if $\MMM$ 
  classifies at least
  one example positively.
  This already suffices to show the stated
  results for $\MM^g_\MAJ$-\HOM{}.
  Moreover, to show the results for
  $\MM{}^g_\MAJ$-\kHOM{}, 
  we additional show that $G$ has a $k$-clique if and only if
  $\MMM$ 
  classifies an example positively that sets at most $k$
  features to $1$.

  $\MMM$ contains the following ensemble elements:
  \begin{itemize}
  \item For every non-edge $uv \notin E(G)$ with $u \neq v$,
    we add $n$ models $M^g_{\{\{f_u,f_v\}\},0}$ to $\MMM$.
  \item For every vertex $v \in V(G)$, we add $1$ model $M^g_{\{\{f_v\}\},1}$ to $\MMM$.
  \item We add $n(\binom{n}{2}-m)-n+2k-1$ models $M^g_{\emptyset,0}$ to $\MMM$.
  \end{itemize}
  Note that each individual model in $\MMM$ can be constructed in
  polynomial time because $\MM^=$ is \dupcc{} and $\MM^\subseteq$ is \piescc{}.  
  Thus, the entire $\MMM$ can also be constructed in polynomial time.
  Moreover, $\MMM(e_0)=0$ because only 
  $n(\binom{n}{2}-m)$ out of a total of $2n(\binom{n}{2}-m) + 2k -1$
  models in $\MMM$ classify $e_0$ positively.
    
  Towards showing the forward direction, let $C=\{v_1,\dotsc,v_k\}$ be
  a $k$-clique of $G$ with $v_i \in V_i$ for every $i \in [k]$. We
  claim that the example $e$ that sets all features in $\SB f_v \SM v
  \in C \SE$ to $1$ and all other features to $0$ satisfies
  $\MMM(e)\neq \MMM(e_0)$. 
  Because $C$ is a clique in $G$, we obtain
  that 
  all $n(\binom{n}{2}-m)$ models from $\MMM$ 
  of structure $M^g_{\{\{f_u,f_v\}\},0}$ 
  $1$-classify $e$. 
  Moreover, because $e$ sets exactly $k$
  features to $1$, it holds that 
  exactly $k$ out of $n$ models from $\MMM$ 
  of structure $M^g_{\{\{f_v\}\},1}$ 
  $1$-classify $e$.
  Therefore, $e$ is classified positively by
  exactly $(\binom{n}{2}-m)+k$ models in $\MMM$ and $e$ is
  classified negatively by exactly
  $n-k+n(\binom{n}{2}-m)-n+2k-1=n(\binom{n}{2}-m)+k-1$ models
  in $\MMM$, which shows that $\MMM(e)=1\neq \MMM(e_0)$.
  
  Towards proving the reverse direction,  
  suppose there exists an example $e$  
  that is classified as $1$ by $\MMM$.  
  Let $F' \subseteq F$ be the set of all features assigned positively by $e$.  
  If there exist distinct vertices $v, u \in V(G)$  
  such that $uv \notin E(G)$,  
  then at least $n$ models of the structure $M^g_{\{\{f_u,f_v\}\},0}$  
  and $n(\binom{n}{2}-m)-n+2k-1$ models of $M^g_{\emptyset,0}$ in $\MMM$  
  would classify $e$ as $0$.  
  This would imply that $e$ is classified as $0$ by $\MMM$,  
  contradicting our assumption.  
  Thus, the set $V_{F'} = \SB v \SM f_v \in F' \SE$  
  must form a clique in $G$.  
  Since $G$ is $k$-colorable, it follows that $|F'| \leq k$.  
  If $|F'| < k$, then there are at least  
  $n-k+1 + n(\binom{n}{2}-m)-n+2k-1 = n(\binom{n}{2}-m) + k$  
  models of the structure $M^g_{\emptyset,0}$  
  or $M^g_{\{\{f_u\}\},1}$ in $\MMM$  
  that classify $e$ as $0$,  
  which again contradicts our assumption that $e$ is classified as $1$ by $\MMM$.  
  Therefore, we conclude that $k = |F'| = |V_{F'}|$,  
  and thus, $V_{F'}$ forms a $k$-clique in $G$.
\end{proof}

\subsection{DTs and their Ensembles}

Here, we present our algorithmic lower bounds for \DT{}s, \DTO{}s, and their ensembles.  
Since \DTO{} and \DTOEO{} are more restrictive variants of \DT{} and \RF{}, respectively,  
we state all results in terms of \DTO{}s and \DTOEO{}.  
We first introduce an auxiliary lemma,  
which will simplify the descriptions of our reductions.
\begin{lemma}\label{lem:DT-construction}
  Let $E \subseteq E(F)$ be a set of examples over a set of features
  $F$ and let $<$ be a total ordering of $F$.  
  In time $\bigoh(|E||F|)$, we can construct a \DTO{} $\TTT_E$  that
  respects $<$, has size at most $2|E||F| + 1$, and has exactly $|E|$ positive leaves,  
  such that for every example $e$ with $F \subseteq F(e)$,  
  it holds that $\TTT_E(e) = 1$ if and only if there exists $e' \in E$ such that $e$ agrees with $e'$.
\end{lemma}
\begin{proof}
  Let $(f_1, \dotsc, f_n)$ denote the sequence of features in $F$, ordered according to $<$.
  For every $e \in E$, we construct a simple \DTO $\TTT_e=(T_e,\lambda_e)$ that
  classifies examples that agrees with $e$ as $1$ and all other examples as $0$.
  $\TTT_e$ has one inner node $t_i^e$ for every $i \in [n]$ with
  $\lambda_e(t_i^e)=f_i$. Moreover, for $i\in [n-1]$, $t_i^e$ has $t_{i+1}^e$
  as its $e(f_i)$-child and a new $0$-leaf as its other
  child. Finally, $t_n^e$ has a new $1$-leaf as its $e(f_n)$-child and
  a $0$-leaf as its other child. Clearly, $\TTT_e$ can be constructed
  in time $\bigoh(|F|)$.
  
  We now construct $\TTT_E$ iteratively starting from $\TTT_\emptyset$
  and adding one example from $E$ at a time (in an arbitrary order).
  We set $\TTT_\emptyset$ to be the \DTO{} that only consists of a
  $0$-leaf. Now to obtain $\TTT_{E'\cup \{e\}}$ from $\TTT_{E'}$
  for some $E'\subseteq E$ and $e \in E \setminus E'$, we do the following. Let
  $l$ be the $0$-leaf of $\TTT_{E'}$ that classifies
  $e$ and let $f_i$ be the feature assigned to the parent of
  $l$. Moreover, let $\TTT_e'$ be the sub-\DTO{} of $\TTT_e$ rooted at
  $t^e_{i+1}$ or if $i=n$ let $\TTT_e'$ be the \DTO{} consisting only of a $1$-leaf.
  Then, $\TTT_{E'\cup \{e\}}$ is obtained from the
  disjoint union of $\TTT_{E'}$ and $\TTT_e'$ after
  identifying the root of $\TTT_e'$ with $l$. 
  Clearly, $\TTT_E$ satisfies the statement of the lemma.
\end{proof}

Next, we show that the class of \DTO{}s satisfies the \dup{} property.  
Moreover, we describe how the model parameters behave as functions of $a$ and $b$.

\begin{lemma}\label{lem:meq_dt}
    \DTO{} is \dup.
    Moreover, for every $a$, $b$, and order $<$, there exists
    $\MM{} \subseteq \DTO{}$,
    such that \MM{} is \dupab{}
    and for every $M \in \MM$, $M$ respects $<$,
    the $\mnlsize$ of $M$ is at most $b$
    and the $\sizeelem$ of $M$ is at most $2ab^2+1$.
\end{lemma}
\begin{proof}
  Let $c \in \{0,1\}$ and let $\FFF$ be a family of sets of features  
  such that $\max_{F \in \FFF} |F| \leq a$ and $|\FFF| \leq b$.
  and let $<$ be an order.
  To prove the statement, it suffices to show that in time  
  $\poly(a + b)$ we can construct  
  a \DTO{} $\TTT$ such that $\TTT$ classifies every example identically to $M^-=_{\FFF,c}$,
  $\TTT$ respects $<$, and satisfy the required size constraints.

  Let $\FFF = \SB F_1, \dots, F_{b'} \SE$ with $b' \leq b$.
  Define $E = \SB e_i \SM i \in [b'] \SE$,  
  where each $e_i$ is an example over $\bigcup \FFF$ that assigns value $1$ 
  only to the features in $F_i$, i.e.,  
  for every $f \in \bigcup \FFF$, we have $e_i(f) = 1 \Leftrightarrow f \in F_i$.

  Let $\TTT_E$ be the \DTO{} obtained from \Cref{lem:DT-construction} using the set $E$ and the order $<$.  
  If $c = 1$, then we define $\TTT = \TTT_E$;  
  otherwise, we obtain $\TTT$ by flipping the label of every leaf in $\TTT_E$.
  Since $|\bigcup \FFF| \leq ab' \leq ab$ and $|E| = |\FFF| \leq b$,  
  it follows that the \sizeelem{} of $\TTT$ is at most $2ab + 1$,  
  and the \mnlsize{} of $\TTT$ is at most $b$.  
  Moreover, $\TTT$ can be computed in time $\poly(a + b)$.
  Clearly, $\TTT$ classifies examples in exactly the same way as $M^=_{\FFF,c}$
  and respects $<$.
\end{proof}

The next theorem immediately follows from~\Cref{lem:meq_dt,lem:new_wh2}.
We note that an analogous theorem for \DT{}s (instead of \DTO{}s)  
can be derived from a result in~\cite[Proposition 6]{BarceloM0S20} for
free binary decision diagrams.

\begin{theorem}\label{th:DT-MLAEX-W2}
    The following holds:
    \begin{itemize}
        \item \DTO{}-\MLAEX{} is \NPh{};
        \item \DTO{}-\MLAEX{}$(\xpsize{})$ is \Wh{2}.
    \end{itemize}
\end{theorem}

The next two theorems cover ensembles of \DTO{}s, which naturally are
harder to explain.

\begin{theorem}\label{th:RF-W1-ES}
  Let $\PP \in \{\LAEX{}, \GAEX{}, \GCEX{}\}$. 
  Then the following holds:
  \begin{itemize}
      \item  \RFO{}-\SPP{}$(\enssize{})$ is \coWh{1};
      \item \RFO{}-\MPP{}$(\enssize{})$ is \coWh{1} even if \xpsize{} is constant;
      \item \RFO{}-\SMLCEX{}$(\enssize{})$ is \Wh{1};
      \item \RFO{}-\MLCEX{}$(\enssize{}+\xpsize{})$ is \Wh{1}.
  \end{itemize}
\end{theorem}
\begin{proof}
  Let $<^\star$ be an arbitrary total order over the features.  
  By \Cref{lem:meq_dt},  
  there exists a class $\MM \subseteq \DTO{}$ such that $\MM$ is \dupc{},  
  and every model in $\MM$ respects $<^\star$.
  From \Cref{lem:new_ens_wh1}, we obtain that:
  $\MM_\MAJ$-\HOM{}$(\enssize{})$ and
  $\MM_\MAJ$-\kHOM{}$(\enssize{} + k)$ are \Wh{1}.
  Since $\MM \subseteq \DTO{}$
  and every model in $\MM$ respects $<^\star$,
  the same hardness results apply to the \RFO{}. 
  From \Cref{lem:HOM-reduction,lem:kHOM-reduction},  
  we obtain the claimed hardness results for the respective problems in the theorem.
\end{proof}

\begin{theorem}\label{th:RF-paraNP}
  Let $\PP \in \{\LAEX{}, \GAEX{}, \GCEX{}\}$. 
  Then the following holds:
    \begin{itemize}
      \item \RFO{}-\SPP{} is \coNPh{} even if  $\mnlsize{} + \sizeelem{}$ is constant;
      \item \RFO{}-\MPP{} is \coNPh{} even if  $\mnlsize{}+ \sizeelem{} + \xpsize{}$ is constant;
      \item \RFO{}-\SMLCEX{} is \NPh{} even if  $\mnlsize{} + \sizeelem{}$ is constant;
      \item \RFO{}-\MLCEX{}$(\xpsize{})$ is \Wh{1} even if $\mnlsize{} + \sizeelem{}$ is constant.
  \end{itemize}
\end{theorem}
\begin{proof}
  Let $<^\star$ be an arbitrary total order over the features.  
  By \Cref{lem:meq_dt},  
  there exists a class $\MM \subseteq \DTO{}$ such that
  $\MM$ is \dupcc{},
  every model in $\MM$ respects $<^\star$ and
  every \DTO{} $\TTT \in \MM$ has $\mnlsize{} \leq 1$ and $\sizeelem{} \leq 5$.
  From \Cref{lem:new_wh1}, we obtain that:
  $\MM_\MAJ$-\HOM{} is \NPh{} even if $\mnlsize{} + \sizeelem{}$ is constant;
  $\MM_\MAJ$-\kHOM{}$(k)$ is \Wh{1} even if $\mnlsize{} + \sizeelem{}$ is constant.
  Since $\MM \subseteq \DTO{}$  and every model in $\MM$ respects $<^\star$,
  the same hardness results apply to the \RFO{}. 
  From \Cref{lem:HOM-reduction,lem:kHOM-reduction},  
  we obtain the claimed hardness results for the respective problems in the theorem.
\end{proof}

The following theorem can be seen as an analogue of \Cref{th:DT-MLAEX-W2} 
for global abductive and global contrastive explanations.  
Interestingly, while for algorithms it was not necessary  
to distinguish between local abductive explanations and global explanations,  
this distinction becomes crucial when establishing lower bounds.  
Moreover, while the following result establishes \W{1}-hardness  
for \DTO{}-\MGAEX{}$(\xpsize{})$ and \DTO{}-\MGCEX{}$(\xpsize{})$,  
this is achieved via fpt-reductions that are not polynomial-time reductions,  
a behaviour rarely seen in natural parameterized problems.  
Thus, while it remains unclear whether these problems are \NPh{},  
the result still shows that they are not solvable in polynomial time unless $\FPT = \W{1}$,  
which is widely believed to be unlikely~\cite{DowneyFellows13}.



\begin{theorem}\label{th:DT-MGAEX-W1}
  \DTO{}-\MGAEX{}$(\xpsize{})$ and \DTO{}-\MGCEX{}$(\xpsize{})$ are \Wh{1}. 
  Moreover, there is no polynomial time algorithm for solving \DTO{}-\MGAEX{} and \DTO{}-\MGCEX{},
  unless $\FPT = \W{1}$.
\end{theorem}
\iflong\begin{proof}\fi
\ifshort\begin{proof}[Proof Sketch]\fi 
  We provide a parameterized reduction from the \textsc{Multicoloured
    Clique} (\MCC)
  problem, which is well-known to be \Wh{1} parameterized by the size
  of the solution. Given an instance $(G,k)$
  of the \MCC{} problem with $k$-partition $(V_1,\dotsc,V_k)$ of $V(G)$, we will construct an equivalent instance
  $(\TTT,c,k)$ of \MGAEX{} in fpt-time. Note that since
  a partial example $e'$ is a global abductive explanation for $c$
  w.r.t. $\TTT$ if and only if $e'$ is a global contrastive
  explanation for $1-c$ w.r.t. $\TTT$, this then also implies the
  statement of the theorem for \MGCEX{}.
  $\TTT$ uses one binary feature~$f_v$ for every $v \in V(G)$
  and $2^{k + \lceil \log (k(k-1))\rceil}$ new auxiliary features.

  For every $i \in [k]$, let $F_i = \SB f_v \SM v \in V_i \SE$.  
  Let $<^\star$ be an arbitrary total order over the features,
  such that for all $i, j \in [k]$, $f \in F_i$, and $f' \in F_j$,  
  if $i < j$, then $f <^\star f'$.
  We start by constructing the \DTO{} $\TTT_{i,j}$ that respects $<^\star$
  and
  for every $i,j \in
  [k]$ with $i< j$ satisfying the
  following: (*) $\TTT_{i,j}(e) = 1$ for an example $e$ if and only if
  either $e(f_v) = 0$ for every $v \in V_i$ or 
  there exists 
  $v\in V_i$ such that $e(f_{v}) = 1$ and $e(f_{u}) = 0$
  for every $u \in (V_i\setminus \{v\}) \cup (N_G(v) \cap V_j)$.
  Let $\TTT_{i}$ be the \DTO{} obtained using
  \Cref{lem:DT-construction} for the order $<^\star$ and 
  for the set of examples $\{e_0\} \cup \SB
  e_v \SM v\in V_i \SE$ defined on the features in $F_i=\SB f_v\SM
  v\in V_i \SE$. Here, $e_0$ is the all-zero example and for every $v
  \in V_i$, $e_v$ is the example that is $1$ only at the feature $f_v$
  and $0$ otherwise.
  Moreover, for every $v\in V_i$, let $\TTT^{v}_{j}$ be the \DTO{} obtained using
  \Cref{lem:DT-construction} for the order $<^\star$ and for the set of examples containing only
  the all-zero example defined on the features in
  $\SB f_{v'}\SM v'\in N_G(v)\cap V_j\SE$.
  Then, $\TTT_{i,j}$ is obtained from $\TTT_i$ after replacing
  the $1$-leaf that classifies $e_v$  with $\TTT^v_j$ for every $v \in
  V_i$.
  Clearly, $\TTT_{i,j}$ satisfies (*) and since 
  $\TTT_{i}$ has at most $|V_i|^2$ inner nodes and $|V_i| + 1$ $1$-leaf nodes,
  and $\TTT^v_{j}$ has at most
  $|V_j|$ inner nodes, we obtain that $\TTT_{i,j}$ has at most
  $\bigoh(|V(G)|^2)$ nodes.

  For an integer $\ell$, we denote by $\DTO(\ell)$ the
  complete \DTO{} of height $\ell$,
  where every inner node is
  assigned to a fresh auxiliary feature and every of the exactly
  $2^\ell$ leaves is a $0$-leaf. 
  Let $\TTT_\Delta$ be the \DTO{} obtained from the disjoint union of
  $\TTT_U=\DL(k)$ and $2^{k}$ copies $\TTT_D^1,\dotsc,
  \TTT_D^{2^{k}}$ of \DTO($\lceil \log (k(k-1)/2)\rceil$) by
  identifying the $i$-th leaf of $\TTT_U$ with the root of $\TTT_D^i$
  for every $i$ with $1 \leq i \leq 2^{k}$; each
  copy is equipped with its own set of fresh features.

  Then,
  $\TTT$ is obtained from $\TTT_\Delta$ after doing the following with
  $\TTT_D^\ell$ for every $\ell \in [2^k]$.
  For every $i,j \in [k]$ with $i< j$,
  we replace a private leaf of $\TTT_D^\ell$ with the \DTO{}
  $\TTT_{i,j}$; note that this is possible because
  $\TTT_D^\ell$ has at least $k(k-1)/2$ leaves. Also note that
  $\TTT$ has size at most 
  $\bigoh(|\TTT_\Delta||V(G)|^2)$
  $=\bigoh(2^{k + \lceil \log (k(k-1))\rceil}|V(G)|^2)=\bigoh(k^2 2^k|V(G)|^2)$.
  This completes the construction of $\TTT$ and we set $c=0$.
  Clearly, $\TTT$ can be constructed from $G$ in fpt-time with respect to $k$,  
  and there exists an order $<'$ (extending $<^\star$) such that $\TTT$ respects $<'$.
  It remains to show that
  $G$ has a $k$-clique if and only if there is a global abductive
  explanation of size at most $k$ for $c$ w.r.t. $\TTT$.
  \iflong 

  Towards showing the forward direction, let $C=\{v_1,\dotsc,v_k\}$ be
  a $k$-clique of $G$ with $v_i \in V_i$ for every $i \in [k]$. We
  claim that $\alpha : \SB f_v \SM v \in V(C) \SE \rightarrow \{1\}$ is a
  global abductive explanation for $c$ w.r.t. $\TTT$, which
  concludes the proof of the forward direction.   
  To see that this is
  indeed the case consider any example $e$ that agrees with $\alpha$,
  i.e., $e$ is $1$ at any feature $f_{v}$ with $v \in V(C)$.
  For this is suffices to show that $\TTT_{i,j}(e)=0$ for every
  $i,j\in [k]$. Because $\TTT_{i,j}$ satisfies (*) and
  because $e(f_{v_i})=1$, it has to
  hold that $e$ is $1$ for at least one feature in
  $\SB f_v \SM v \in N_G(v_i)\cap V_j\SE$. But this clearly holds
  because $C$ is a $k$-clique in $G$.

  Towards showing the reverse direction, let $\alpha : C' \rightarrow
  \{0,1\}$ be a global abductive explanation of size at most $k$ for
  $c$ w.r.t. $\TTT$. 
  We claim that $C=\SB v \SM f_v \in C'\SE$ is a
  $k$-clique of $G$.
  Let $\TTT_\alpha$ be the partial \DTO{} obtained from
  $\TTT$ after removing all nodes that can never be reached by an
  example that is compatible with $\alpha$, i.e., we obtain
  $\TTT_\alpha$  from $\TTT$ by removing the subtree rooted at the $1-\alpha(t)$-child for every node $t$ of $\TTT$
  with $\lambda(t) \in C'$. Then, $\TTT_\alpha$ contains
  only $0$-leaves. We first show that there is an $\ell \in
  [2^{k}]$
  such that $\TTT_\alpha$ contains $\TTT_D^\ell$ completely, i.e., this in particular means that $C'$ contains no feature of $\TTT_D^\ell$. To see this, let $x$ be the number of features in
  $C'$ that are assigned to a node of $\TTT_U$. Then, $\TTT_\alpha$
  contains the root of at least $2^{k-x}$ \DTO{}s
  $\TTT_D^i$. Moreover, since $2^{k-x}\geq k-x$ there is at least
  one $\TTT_D^i$ say $\TTT_D^\ell$, whose associated features are not in
  $C'$.  Therefore, for every $i,j \in [k]$ with $i \neq j$, $\TTT_\alpha$ contains at least the root of $\TTT_{i,j}$. Since $\TTT_{i,j}$ satisfies (*), it follows that for every $\ell \in [k]$ there is $v_\ell \in V_\ell$ such that $\alpha(f_{v_{\ell}})=1$. Because $|C'|\leq k$, we obtain that $C'$ contains exactly one feature $f_{v_\ell} \in F_\ell$ for every $\ell$. Finally, using again that $\TTT_{i,j}$ satisfies (*), we obtain that $v_{i}$ and $v_{j}$ are adjacent in $G$, showing that $\{v_1,\dotsc,v_k\}$ is a $k$-clique of $G$.
 \fi
\end{proof}

\subsection{DSs, DLs  and their Ensembles}

Here, we establish hardness results for \DS{}s, \DL{}s,  
and their ensembles.  
Interestingly, there is no major distinction between \DS{}s and \DL{}s regarding explainability,  
and both are considerably harder to explain than \DT{}s.  
In particular, we show that the class of \DS{}s and \DL{}s satisfies the \pies{} property,  
and we analyse how their parameters behave as functions of $a$ and $b$.

\begin{lemma}\label{lem:mss_ds_dl}
    \DS{} and \DL are \pies{}.
    Moreover, for every $a$ and $b$, there exists
    $\MM{} \subseteq \DS{}$ ($\MM{} \subseteq \DL{}$),
    such that \MM{} is \piesab{}
    and for every $M \in \MM$,
    the $\termsize{}$ of $M$ is at most $a$,
    the $\termselem{}$ of $M$ is at most $b$ and
    the $\sizeelem{}$ of $M$ is at most $ab + b +  1$.
\end{lemma}
\begin{proof}
  Let $c \in \{0,1\}$ and let $\FFF$ be a family of sets of features  
  such that $\max_{F \in \FFF} |F| \leq a$ and $|\FFF| \leq b$.
  To prove the statement, it suffices to show that in time  
  $\poly(a + b)$ we can construct  
  a \DS{} $\ds$ and a \DL{} $\dl$ such that both models classify every example identically to $M^\subseteq_{\FFF,c}$, and
  both satisfy the required size constraints.

  Let $\FFF = \SB F_1, \dots, F_{b'} \SE$ with $b' \leq b$.  
  For each $i \in [b']$, let the term $t_i$ consist of positive literals over the features in $F_i$, i.e.,  
  $t_i = \SB (f = 1) \SM f \in F_i \SE$.
  We define 
  \DS $\ds = (\SB t_i \SM i \in [b'] \SE,\ 1 - c)$ and
  \DL $\dl = ((t_1, c), \dots, (t_{b'}, c), (\emptyset, 1-c))$.

  Clearly, both $\ds$ and $\dl$ can be constructed in time  
  $\poly(a + b)$. Moreover, both models classify  
  examples exactly as $M^\subseteq_{\FFF,c}$ does.
  It remains to verify the size bounds.
  The $\termsize{}$ of $L$ ($S$) is at most $\max_{i \in [b']}|t_i| = \max_{i \in [b']}|F_i| \leq a$,
  the $\termselem{}$ of $L$ ($S$) is at most $|\FFF|+1 = b' +1\leq b+1$ and
  the $\sizeelem{}$ of $L$ ($S$) is at most $(\sum_{i \in [b']}|t_i| + 1)  \leq (a+1)b' + 1 \leq ab + b + 1$.
\end{proof}

The proof of the next three theorems makes strong use of \Cref{lem:mss_ds_dl}.
\begin{theorem}\label{th:DS-SMLCEX-W1}
  Let $\MM \in \{\DS,\DL\}$ and $\PP \in \{\MLAEX, \MGAEX, \MGCEX, \MLCEX\}$. 
  Then the following holds:
  \begin{itemize}
      \item $\MM$-$\PP$ is \NPh{};
      \item $\MM$-$\PP(\xpsize{})$ is \Wh{2}.
  \end{itemize}
\end{theorem}
\begin{proof}
  This follows immediately from~\Cref{lem:mss_ds_dl,lem:new_wh2}.
\end{proof}

\begin{theorem}\label{th:DSE-SMLCEX-W1}
  Let $\MM \in \{\DS,\DL\}$. 
  Then,
  $\MM_\MAJ$-\MLCEX{}$(\enssize{}+\xpsize{})$ is \Wh{1} even if \termsize{} is constant.
\end{theorem}
\begin{proof}
  By \Cref{lem:mss_ds_dl},  
  there exists a class $\MM' \subseteq \DS{}$
  such that $\MM'$ is \piesc{},  
  and every \DS{} $\ds \in \MM'$ has $\termsize{} \leq 2$.
  From \Cref{lem:new_ens_wh1}, we obtain that
  $\MM'_\MAJ$-\kHOM{}$(\enssize{} + k)$ is \Wh{1} even if $\termsize{}$ is constant.
  Since $\MM' \subseteq \DS{}$, the same hardness results apply to the \DS{}.
  From \Cref{lem:kHOM-reduction},  
  we obtain the claimed \DS{} hardness results for the respective problems in the theorem.
  Analogously, the corresponding results also hold for \DL{} by a similar argument.
\end{proof}
\begin{theorem}\label{th:DSE-paraNP} 
  Let $\MM \in \{\DS,\DL\}$ and let $\PP \in \{\LAEX, \GAEX, \GCEX\}$. 
  Then the following holds:
  \begin{itemize}
      \item $\MM_\MAJ$-\SPP{} is \coNPh{} even if $\termselem{} + \termsize{}$ is constant;
      \item   $\MM_\MAJ$-\MPP{} is \coNPh{} even if $\termselem{}+ \termsize{} + \xpsize{}$ is constant;
      \item $\MM_\MAJ$-\SMLCEX{} is \NPh{} even if $\termselem{} + \termsize{}$ is constant;
      \item $\MM_\MAJ$-\MLCEX{}$(\xpsize{})$ is \Wh{1} even if $\termselem{} + \termsize{}$ is constant.
  \end{itemize}
\end{theorem}
\begin{proof}
  By \Cref{lem:mss_ds_dl},  
  there exists a class $\MM' \subseteq \DS{}$
  such that $\MM'$ is \piescc{},  
  and every \DS{} $\ds \in \MM'$ has $\termsize{} \leq 2$
  and $\termselem{} \leq 1$.
  From \Cref{lem:new_wh1}, we obtain that:
  $\MM'_\MAJ$-\HOM{} is \NPh{} even if $\termsize{} + \termselem{}$ is constant;
  $\MM'_\MAJ$-\kHOM{}$(k)$ is \Wh{1} even if $\termsize{} + \termselem{}$ is constant.
  Since $\MM' \subseteq \DS{}$, the same hardness results apply to the \DS{}.
  From \Cref{lem:HOM-reduction,lem:kHOM-reduction},  
  we obtain the claimed \DS{} hardness results for the respective problems in the theorem.
  Analogously, the corresponding results also hold for \DL{} by a similar argument.
\end{proof}

Finally, we are able to provide hardness even for single DSs and DTs,
where every term has size at most $3$. To obtain this result, we will
reduce from the following problem, which is well-known to be \coNPh{}.

\newcommand{\TAU}{\textsc{Taut}}
\iflong
\pbDef{{\sc $3$-DNF Tautology} (\TAU)}
{A $3$-DNF formula $\psi$.}
{Is $\psi$ satisfied for all assignments?}
\fi
Additionally, we note that related but orthogonal hardness results  
for less restrictive settings are established in \Cref{th:DS-SMLCEX-W1}.

\begin{theorem}\label{th:DS-paraCoNP}
  Let $\MM \in \{\DS,\DL\}$ and let $\PP \in \{\LAEX, \GAEX, \GCEX\}$. 
    Then the following holds:
  \begin{itemize}
      \item   \MM-\SPP{} is \coNPh{} even if \termsize{} is constant;
      \item \MM-\MPP{} is \coNPh{} even if $\termsize{} + \xpsize{}$ is constant.
      \item \MM-\SMLCEX{} is \NPh{} even if \termsize{} is constant;
  \end{itemize}
\end{theorem}
\iflong\begin{proof}
    Because every \DS{} can be easily transformed into an equivalent
    \DL{} that shares all of the same parameters, it is sufficient to
    show the theorem for the case that $\MM=\DS$.
        
    First, we give a polynomial-time reduction from \TAU{} to
    \DS-\HOM{}. 
    Let $\psi$ be the $3$-DNF formula given as an instance of \TAU.
    W.l.o.g., we can assume that the all-zero assignment satisfies
    $\psi$, since otherwise $\psi$ is a trivial no-instance.
    Then, the \DS{} $\ds = (T_{\psi}, 0)$, where $T_{\psi}$ is the set
    of terms of $\psi$, is the constructed instance of
    \DS-\HOM{}. Since the reduction is clearly polynomial, it only
    remains to show that $\psi$ is a tautology if and only if $\ds$
    classifies all examples in the same manner as the all-zero
    example, i.e., positively. But this is easily seen to hold, which
    concludes the reduction from \TAU{} to \DS-\HOM{}.
    All statements in the theorem now follow from 
    \Cref{lem:HOM-reduction}.
\end{proof}\fi

\section{Conclusion}\label{sec:conclusion}

We have charted the computational complexity landscape of explanation problems across a range of machine learning models, from decision trees to their ensembles. Our journey through this landscape has revealed both expected patterns and surprising anomalies that challenge our intuitions about model explainability.

Perhaps most striking among our findings is the stark contrast between decision trees and their seemingly simpler cousins, decision sets and decision lists. While \DT{}s admit polynomial-time algorithms for all
variants of local contrastive explanations (\LCEX{}), the same
problems become intractable for \DS{}s and \DL{}s. This
counter-intuitive result suggests that the structure of a model, how it
organises its decision logic, can matter more for explainability than
the apparent simplicity of its rules. Nevertheless, the picture is not entirely bleak: we identify specific conditions under which even these harder models become tractable, such as our fixed-parameter tractable algorithms for \DL{}-\MLCEX{} when certain parameters are bounded (\Cref{cor:DS-MLCEX-FPT-const-ens}).

The theoretical frameworks we have developed extend well beyond the specific models studied here. Our algorithmic meta-theorem, built on Boolean circuits and extended monadic second-order logic, provides a
unified approach to obtaining tractability results. Similarly, our
hardness framework, based on the set-modeling and subset-modeling
properties, offers a systematic way to establish intractability. These
frameworks have already proven their worth: in Part II of this paper,
they enabled us to map out an almost complete complexity landscape for
explanation problems on binary decision diagrams. The generality of these tools suggests they will find applications in analysing future model classes as they emerge.

Our analysis also revealed several noteworthy complexity-theoretic
phenomena. The problem \DT{}-\MGAEX{} occupies an unusual position in
the complexity hierarchy (\Cref{th:DT-MGAEX-W1}): we prove it is
\Wh{1}, placing it beyond polynomial-time solvability unless $\FPT =
\W{1}$, yet we cannot establish whether it is \NP{}-hard. Such
problems, known to be hard for one complexity class but not proven
hard for a seemingly weaker one, are rare and highlight the subtle
gaps in our understanding of the complexity hierarchy. We also showed
that ordered decision trees do not offer computational advantages when
it comes to explanability compared to their unordered
counterparts. This is rather surprising given that they are
significantly more restrictive. It also serves as an important
prerequisite to the analysis of binary decision diagrams, which is
left to Part II of this paper.


The close connection we uncovered between explanation problems and the homogeneous model problem, determining whether a model classifies all
examples identically, provides a new lens through which to view explainability. This connection not only simplified many of our proofs but also suggests a deeper relationship between a model's ability to discriminate between examples and the difficulty of explaining its decisions.

Several open problems remain. We still do not know how to count
explanations efficiently: given a model and an example, how many
distinct minimal explanations exist? The enumeration problem is
similarly unexplored. Constrained explanations, where we want
explanations with specific properties beyond just minimality, present
another set of challenges \cite{BarceloM0S20}. For weighted ensembles,
the situation becomes murkier still. While our hardness results will
likely transfer over, we do not yet understand when weighted voting might enable efficient explanation algorithms, particularly for polynomially-bounded weights.

Our work contributes to the theoretical foundations of explainable AI at a time when such foundations are sorely needed. As regulatory frameworks increasingly demand transparency in automated decision-making, understanding the fundamental limits of explainability becomes crucial. While we must be cautious about drawing immediate practical conclusions from complexity-theoretic results, our findings do suggest that the choice of model architecture has profound implications for the feasibility of generating explanations at scale. 

The frameworks and results presented here serve two purposes: they provide tools for analysing future model classes, and they remind us that computational complexity shapes not just what we can efficiently compute, but what we can efficiently explain. As AI systems continue to influence critical decisions, this distinction becomes increasingly important. Even models designed to be ``interpretable'' can require exponential time to explain, a sobering reminder that transparency in AI remains as much a computational challenge as a design principle.

\section*{Acknowledgements}
Stefan Szeider acknowledges support by the Austrian Science Fund (FWF)
within the projects 10.55776/P36688,
10.55776/P36420, and 10.55776/COE12.
Sebastian Ordyniak was supported by the Engineering
and Physical Sciences Research Council (EPSRC) (Project EP/V00252X/1).


\end{document}